%% file: main.tex
\documentclass{article} %
\usepackage[utf8]{inputenc}

\usepackage{etoolbox}
\newcommand{\arxiv}[1]{\iftoggle{neurips}{}{#1}}
\newcommand{\neurips}[1]{\iftoggle{neurips}{#1}{}}
\newtoggle{neurips}

\arxiv{\usepackage{txie}}
\neurips{
\usepackage{txie}
\usepackage[final, nonatbib]{neurips_2025}
}

\input{macro}

\input{command}

\usepackage{xargs}

\input{macro_fan}

\arxiv{
\usepackage{geometry}
\geometry{letterpaper, margin=1in}
}

\usepackage{parskip}
\makeatletter
\def\thm@space@setup{%
  \thm@preskip=8pt plus 2pt minus 4pt
  \thm@postskip=\thm@preskip
}
\makeatother

\let\oldparagraph\paragraph
\renewcommand{\paragraph}[1]{\oldparagraph{#1.}}

\usepackage{booktabs}
\usepackage[most]{tcolorbox}
\usepackage{adjustbox}

\usepackage[T1]{fontenc}
\usepackage[scaled=.91]{helvet}
\usepackage{inconsolata}
\usepackage{microtype}

\addtocontents{toc}{\protect\setcounter{tocdepth}{0}}

\allowdisplaybreaks

\title{Trajectory Bellman Residual Minimization: \\
A Simple Value-Based Method for LLM Reasoning}

\arxiv{
    \author{Yurun Yuan\footnotemark[2] \qquad Fan Chen\footnotemark[3] \qquad
    Zeyu Jia\footnotemark[3]\qquad Alexander Rakhlin\footnotemark[1]~\footnotemark[3]\qquad  Tengyang Xie\footnotemark[1]~\footnotemark[2]
    \\\\
    \footnotemark[2] UW-Madison \qquad\qquad \footnotemark[3] MIT}
}
\neurips{
    \author{%
    Yurun Yuan$^1$ \quad
    Fan Chen$^2$ \quad
    Zeyu Jia$^2$ \quad
    Alexander Rakhlin$^{2}$\thanks{Corresponding authors.} \quad
    Tengyang Xie$^{1}$\footnotemark[1] \\
    \vspace{1mm} \\
    $^1$University of Wisconsin--Madison \qquad $^2$Massachusetts Institute of Technology \\
    \texttt{\{\href{mailto:yurun\_yuan@cs.wisc.edu}{{yurun\_yuan}},\href{mailto:tx@cs.wisc.edu}{tx}\}@cs.wisc.edu} \\
    \texttt{\{\href{mailto:fanchen@mit.edu}{fanchen},\href{mailto:zyjia@mit.edu}{zyjia},\href{mailto:rakhlin@mit.edu}{rakhlin}\}@mit.edu}
    }
}

\date{\today}

\newcommand{\epsstat}{\varepsilon_\stat}

\newcommand{\Concinf}[1]{C_{\mathrm{conc},\infty}(#1)}
\newcommand{\Conchis}[1]{C_{\mathrm{conc}}(#1)}

\begin{document}

\maketitle

\arxiv{
\def\thefootnote{}
\footnotetext{Email: \texttt{\{\href{mailto:yurun\_yuan@cs.wisc.edu}{{yurun\_yuan}},\href{mailto:tx@cs.wisc.edu}{tx}\}@cs.wisc.edu}, \texttt{\{\href{mailto:fanchen@mit.edu}{fanchen},\href{mailto:zyjia@mit.edu}{zyjia},\href{mailto:rakhlin@mit.edu}{rakhlin}\}@mit.edu}}
\def\thefootnote{$*$}
\footnotetext{Corresponding authors}
\def\thefootnote{\arabic{footnote}}
}

\begin{abstract}
Policy-based methods currently dominate reinforcement learning (RL) pipelines for large language model (LLM) reasoning, leaving value-based approaches largely unexplored. We revisit the classical paradigm of Bellman Residual Minimization and introduce Trajectory Bellman Residual Minimization (TBRM), an algorithm that naturally adapts this idea to LLMs, yielding a simple yet effective off-policy algorithm that optimizes a single trajectory-level Bellman objective using the model's own logits as $Q$-values. TBRM removes the need for critics, importance-sampling ratios, or clipping, and can operate with only one rollout per prompt. We prove convergence to the near-optimal KL-regularized policy from arbitrary off-policy data via an improved change-of-trajectory-measure analysis. Experiments on standard mathematical-reasoning benchmarks show that TBRM matches or surpasses policy-based baselines, like PPO and GRPO, with comparable or lower computational and memory overhead. Our results indicate that value-based RL might be a principled and efficient alternative for enhancing reasoning capabilities in LLMs. The codebase for TBRM is publicly available at \url{https://github.com/rlx-lab/TBRM}.
\end{abstract}

\section{Introduction}
Large language models (LLMs) have become the de-facto backbone for modern natural-language understanding and generation \citep{brown2020language,ouyang2022training}. While ever-larger pre-training corpora push the frontier of \emph{knowledge}, high-value downstream usage increasingly hinges on \emph{reasoning}: the capacity to carry out multi-step thinking, apply abstract rules to practical situations, and generalize from observed patterns to solve complex, structured problems. Reinforcement learning (RL) with verifiable, outcome-based rewards has emerged as a powerful paradigm for enhancing this reasoning capability in LLMs \citep{guo2025deepseek}, especially for mathematical problem-solving where correctness can be objectively determined.

Recent advances in LLM post-training for mathematical reasoning have primarily employed policy-based variants---Proximal Policy Optimization \citep[PPO;][]{schulman2017proximal} and Group Relative Policy Optimization \citep[GRPO;][]{shao2024deepseekmath,guo2025deepseek}---which optimize policies to maximize objective rewards that indicate successful task completion. These approaches leverage the clear evaluation criteria of mathematical tasks, where responses can be automatically verified as correct or incorrect without requiring human judgment.

Despite their empirical success, policy-based methods face several practical challenges. They typically require fresh on-policy rollouts from the current model, increasing computational demands. They often rely on additional components like critic models, advantage normalization, and clipping mechanisms, adding complexity to implementation and tuning. Moreover, for token-level decisions, these methods (1) simplistically attribute outcome-based rewards (e.g., correctness of the entire response) to individual tokens, often assigning credit primarily to the final token, and (2) bootstrap advantages with truncated rollout horizons when additional critic models are used, potentially compromising effective credit assignment during training.

Classical RL offers a compelling alternative: \emph{value-based} methods learn an action-value function \(Q\) and act, for instance, greedily with respect to it. In the LLM setting, each token is an action, and the model's logits (that is, raw network outputs) naturally provide a parametric family for \(Q\), under the KL-regularized RL framework \citep[e.g.,][]{schulman2017equivalence}.
Despite this natural alignment and the inherent strengths of value-based methods for LLMs, their application to LLMs has been limited, with policy-based techniques being more prevalent. One potential reason, we conjecture, is the perceived difficulty in reconciling traditional iterative bootstrapping in value-based RL \citep[e.g., the Q-learning family,][]{watkins1992q,mnih2015human} with the scale of LLM training, as iterative-style algorithms are typically less stable than their optimization-style counterparts (e.g., policy-based methods).

Our starting point is \emph{Bellman Residual Minimization} \citep[BRM;][]{schweitzer1985generalized,baird1995residual}, a decades-old idea that fits the Q-function by directly minimizing its Bellman residual, designed for deterministic environments (which LLMs naturally are).
By leveraging the recent theoretical advances in trajectory-level change of measure \citep{jia2025we}, we recognize that BRM can be lifted from the \emph{step (token)} to the \emph{trajectory} level: we square a single residual spanning the whole rollout and regress the model's logits onto it.
This approach eliminates the aforementioned per-step-signal barrier, removes the need for critics, importance weights, and clipping, and provably maintains fully off-policy optimization due to the value-based nature of the algorithm.
The resulting algorithm is \emph{Trajectory BRM} (TBRM), which builds on classical BRM. Below we state our contributions, focused on  theoretical analysis and extensive experiments on \arxiv{mathematical}\neurips{math} reasoning tasks.

\subsection{Our Results}\label{sec:intro:our_results}

\begin{enumerate}[leftmargin=*]
    \item \textbf{Algorithm.}  Building explicitly on the classical idea of Bellman Residual Minimization, we present \emph{TBRM}, a single-objective, off-policy algorithm that fits the trajectory-level Bellman residual using the LLM logits as $Q$-values. TBRM dispenses with critics, advantage estimates, importance ratios, or clipping, and can operate with only \emph{one rollout per prompt}, while scaling effectively with multiple rollouts in practice. 
    \item \textbf{Theory.}
    We prove that, under standard realizability assumptions, the algorithm converges to the \emph{optimal KL-regularized policy} even when training data are generated by \emph{arbitrary} behavior policies in deterministic environments (such as LLMs).
    Our results build upon the recent change-of-trajectory-measure result of \citet{jia2025we}. We  significantly simplify that proof and improve the rate of convergence in terms of horizon factors.
    Overall, our results offer a theoretically grounded alternative to popular (yet ad-hoc) methods like GRPO.
    \item \textbf{Experiments.}  On six mathematical-reasoning benchmarks---namely AIME24/25, AMC23, MATH500, Minerva-Math, and OlympiadBench---TBRM performs on par with or better than PPO and GRPO baselines. Notably, TBRM achieves up to \(30.5\%\) accuracy on AIME24 with Qwen2.5-Math-7B. Compared to GRPO, it improves the average benchmark score by \(1.3\%\) absolute, while under comparable conditions to PPO, it achieves better performance with \(22.5\%\) less training time and \(33\%\) lower GPU memory. We further demonstrate that TBRM benefits from additional rollouts and the model learns emergent reasoning patterns, such as verification, backtracking, and decomposition, that align with human mathematical practice.
\end{enumerate}

Collectively, our findings suggest that value-based approaches offer a compelling alternative to policy gradient methods especially for enhancing mathematical reasoning capabilities in LLMs. By re-casting value learning at the trajectory level, TBRM provides a \emph{principled, efficient, and theoretically grounded} approach for improving performance on mathematical reasoning tasks while dramatically reducing computational requirements.

\section{Preliminaries}

This section provides the necessary background for our work. We first review the fundamentals of KL-regularized RL (\cref{sec:kl_rl}), then discuss prominent reinforcement learning algorithms applied to large language models (\cref{sec:rl_llm}), and finally introduce the autoregressive function approximation framework (\cref{sec:arfa}) that serves as the foundation for our proposed approach.

\subsection{KL-Regularized Reinforcement Learning}\label{sec:kl_rl}

Reinforcement learning (RL) provides a framework for sequential decision-making problems where an agent interacts with an environment to maximize cumulative rewards. In the context of Markov Decision Processes (MDPs), which provide the theoretical foundation for RL, we consider an episodic finite-horizon framework. Formally, a horizon-$H$ episodic MDP $M=(H,\Scal, \Acal, \Pcal, r, \rho)$ consists of a (potentially very large) state space $\Scal$, an action space $\Acal$, a probability transition function $\Pcal: \Scal \times \Acal \to \Delta(\Scal)$, a reward function $r: \Scal \times \Acal \to \mathbb{R}$, and an initial state distribution $\rho \in \Delta(\Scal)$. The state space is typically layered such that $\Scal = \Scal_1 \cup \Scal_2 \cup \cdots \cup \Scal_H$, where $\Scal_h$ is the set of states reachable at step $h$. A policy $\pi: \Scal \to \Delta(\Acal)$ maps states to distributions over actions and induces a distribution over trajectories $\tau = (s_1, a_1, \ldots, s_H, a_H)$ and rewards $(r_1, \ldots, r_H)$, where the initial state is sampled as $s_1 \sim \rho$, and for $h=1,\ldots,H$: $a_h \sim \pi(s_h)$, $r_h = r(s_h, a_h)$, and $s_{h+1} \sim \Pcal(s_h, a_h)$. We let $\E_{\tau \sim \pi}[\cdot]$ and $\mathbb{P}_{\tau \sim \pi}[\cdot]$ denote expectation and probability under this process, and $\E_{\pi}[\cdot]$ and $\mathbb{P}_{\pi}[\cdot]$ for brevity when $\tau$ is not explicitly mentioned.

For any policy $\pi$, we define the occupancy measures that characterize the probabilities of visiting states and selecting actions when following $\pi$. Specifically, the state occupancy measure $d^{\pi}(s_h) \coloneqq \mathbb{P}_{\tau \sim \pi}[s_h \in \tau]$ represents the probability of visiting state $s_h \in \Scal_h$ under policy $\pi$. Similarly, the state-action occupancy measure $d^{\pi}(s_h, a_h) \coloneqq \mathbb{P}_{\tau \sim \pi}[(s_h, a_h) \in \tau]$ gives the probability of the state-action pair $(s_h, a_h)$ occurring in a trajectory. We also define the trajectory occupancy measure $d^{\pi}(\tau) \coloneqq \mathbb{P}_{\tau' \sim \pi}[\tau' = \tau]$, which is the probability of generating the exact trajectory $\tau$ when following policy $\pi$. It is important to distinguish between $d^{\pi}(\tau)$ and $\pi(\tau) \coloneqq \prod_{(s_h,a_h) \in \tau} \pi(a_h \mid s_h)$, as they differ when the transition dynamics are stochastic.

In standard RL, the objective is to find a policy $\pi$ that maximizes the expected cumulative reward $J(\pi) = \E_{\tau \sim \pi}[r(\tau)]$, where $r(\tau) = \sum_{h=1}^H r(s_h,a_h)$.
In many practical applications, particularly in the context of large language models, it is beneficial to incorporate a regularization term that encourages the learned policy to stay close to a reference policy $\piref$. This leads to the KL-regularized RL objective \citep{ziebart2008maximum,ziebart2010modeling,neu2017unified,ouyang2022training}
\begin{align*}\neurips{\textstyle}
J_{\beta}(\pi) = \E_{\tau \sim \pi}[r(\tau)] - \beta \cdot \E_{\tau \sim \pi}\left[\log\frac{\pi(\tau)}{\piref(\tau)}\right]
= \E_{\tau \sim \pi} \left[ \sum_{h=1}^H \left( r(s_h,a_h) - \beta \log\frac{\pi(a_h\mid s_h)}{\piref(a_h\mid s_h)} \right) \right],
\end{align*}
where $\beta > 0$ is a regularization parameter that controls the strength of the penalty $D_{\text{KL}}(\pi \| \piref) = \E_{\tau \sim \pi}\left[\log\frac{\pi(\tau)}{\piref(\tau)}\right]$, known as the Kullback-Leibler divergence. %

\subsection{Reinforcement Learning Methods for Large Language Models}\label{sec:rl_llm}

In this section, we briefly review popular reinforcement learning methods for large language models. For the sake of generality, we continue to use the terminology of MDPs (i.e., we use $s$ to represent state and $a$ to represent action). This terminology naturally encompasses the case of large language models in both single-turn and multi-turn interaction settings.

In the single-turn setting where $x \sim \rho$ denotes the input prompt and $y_1, y_2, \dotsc, y_H$ denote the output tokens, we can define $s_1 \coloneqq x$ and $s_h \coloneqq (x,y_1,\dotsc,y_{h-1})$ for $h > 1$, with $a_h \coloneqq y_h$ for $h=1,\dotsc,H$. In the multi-turn setting, which consists of multiple interaction turns $(x^\iter{1}, y_{1:H}^\iter{1})$, $(x^\iter{2}, y_{1:H}^\iter{2})$, and so forth, we can adapt the transition function accordingly. Here, $y_{1:H}^\iter{i}$ is a shorthand notation for the sequence of tokens $y_1^\iter{i}, y_2^\iter{i}, \dotsc, y_H^\iter{i}$ in the $i$-th turn. For instance, if a state-action pair $(s,a)$ contains the complete response for one turn (e.g., in a conversation with three or more turns), where $s = (x^\iter{1}, y_{1:H}^\iter{1}, x^\iter{2}, y_{1:H-1}^\iter{2})$ and $a = y_H^\iter{2}$, the next state would transition to $s' = (x^\iter{1}, y_{1:H}^\iter{1}, x^\iter{2}, y_{1:H}^\iter{2}, x^\iter{3})$, rather than simply concatenating the previous state and action as in the single-turn case.

\paragraph{Proximal Policy Optimization (PPO)} PPO \citep{schulman2017proximal} introduces a clipped surrogate objective to constrain policy updates:
\begin{align*}\neurips{\textstyle}
\mathcal{J}^{\sf PPO}(\theta) = \E_{(s_h, a_h) \sim \pi_{\theta_{\sf old}}} \left[ \min\left( \frac{\pi_\theta(a_h\mid s_h)}{\pi_{\theta_{\sf old}}(a_h\mid s_h)} \widehat A_h(s_h,a_h), \text{clip}\left( \frac{\pi_\theta(a_h\mid s_h)}{\pi_{\theta_{\sf old}}(a_h\mid s_h)}, 1-\varepsilon, 1+\varepsilon\right) \widehat A_h(s_h,a_h) \right) \right],
\end{align*}
where $\widehat A_h$ is the advantage estimate, and $\varepsilon$ is a hyperparameter. The advantage $\widehat A_h$ is typically computed using Generalized Advantage Estimation \citep[GAE;][]{schulman2015high}: $\widehat A_h(s_h,a_h) = \sum_{l} (\lambda)^l \delta_{t+l}$, where $\delta_h = r_h + V_\phi(s_{h+1}) - V_\phi(s_h)$ is the temporal difference error and $V_\phi$ is an estimate of the value function of the KL regularized reward $r(s_h,a_h) - \beta \log\frac{\pi(a_h\mid s_h)}{\piref(a_h\mid s_h)}$.

For LLMs, PPO has been widely used for enhancing mathematical reasoning capabilities, where objective rewards signal the correctness of the model's solutions.

\paragraph{Group Relative Policy Optimization (GRPO)} GRPO \citep{shao2024deepseekmath} is a policy-based method that, in practical implementations for LLMs like DeepSeek-R1, samples $G$ responses $o^\iter{1}, \dots, o^\iter{G}$ for each prompt $x$ and computes advantages by normalizing rewards within each prompt group.
In the MDP terminology above, this corresponds to: $s_1^\iter{i} = x \sim \rho$, $o^\iter{i} = (a_1^\iter{i},a_2^\iter{i},\dotsc,a_{|o^\iter{i}|^\iter{i}})$, $s_h^\iter{i} = (x,a_1^\iter{i},\dotsc,a_{h-1}^\iter{i})$, and $r(x,o^\iter{i}) = \sum_{h} r(s_h^\iter{i},a_h^\iter{i})$.
The advantage for the $i$-th response $o^\iter{i}$ (and implicitly for each token within that response) is computed as:
\arxiv{\begin{align*}
    \widehat A^\iter{i} = \frac{r(x, o^\iter{i}) - \operatorname{mean}(\{r(x, o^\iter{1}), \dots, r(x, o^\iter{G})\})}{\operatorname{std}(\{r(x, o^\iter{1}), \dots, r(x, o^\iter{G})\}),}
\end{align*}}
\neurips{$\widehat A^\iter{i} = \frac{r(x, o^\iter{i}) - \operatorname{mean}(\{r(x, o^\iter{1}), \dots, r(x, o^\iter{G})\})}{\operatorname{std}(\{r(x, o^\iter{1}), \dots, r(x, o^\iter{G})\})}$,}
where $r(x, o^\iter{i})$ is the outcome for response $o^\iter{i}$ to prompt $x$ as we defined above. This response-level advantage $\widehat A^\iter{i}$ is then used to replace the step-wise advantage function $\widehat A_h(s_h,a_h)$ in the PPO objective $\mathcal{J}^{\sf PPO}$, but then GRPO objective accommodates the KL-regularization at the end:
\begin{align*}
    \mathcal{J}^{\sf GRPO}(\theta) &= \mathbb{E}_{x \sim \rho, \{o^\iter{i}\}_{i=1}^G \sim \pi_{\theta_{\sf old}}(\cdot \mid x)} \Bigg[ \frac{1}{G} \sum_{i=1}^G \frac{1}{|o_i|} \sum_{h=1}^{|o_i|} \Bigg\{ \\
    & \quad \min \left[ \frac{\pi_\theta(a_h^\iter{i} \mid s_h^\iter{i})}{\pi_{\theta_{\sf old}}(a_h^\iter{i} \mid s_h^\iter{i})} \widehat A^\iter{i}, \text{clip}\left(\frac{\pi_\theta(a_h^\iter{i} \mid s_h^\iter{i})}{\pi_{\theta_{\sf old}}(a_h^\iter{i} \mid s_h^\iter{i})}, 1-\epsilon, 1+\epsilon \right) \widehat A^\iter{i} \right] - \beta D_{\text{KL}}(\pi \| \piref) \Bigg\}\Bigg].
\end{align*}
The normalization by mean and standard deviation is intended to stabilize training by reducing variance. GRPO is often considered a simpler alternative to PPO for post-training LLMs. This is partly because PPO typically involves training a separate critic network and incorporates more complex mechanisms for policy updates. In the context of LLMs, the full complexity of PPO might not always be necessary, leading to the adoption of more streamlined policy gradient methods like GRPO.

\subsection{Autoregressive Function Approximation}\label{sec:arfa}

Having established the principles of KL-regularized RL and reviewed current RL methods for LLMs, we now introduce the key formulation that bridges these concepts: autoregressive function approximation. This framework allows us to naturally parameterize value functions and policies using autoregressive models like LLMs, which is central to our proposed method.

Note that KL-regularized RL has been widely studied in classical RL literature \citep[e.g.,][]{schulman2017equivalence,nachum2017bridging,haarnoja2018soft} and the similar idea of autoregressive function approximation has also been introduced by \citet{guo2022efficient}. This subsection should be viewed as a discussion of preliminary background and unified notations that will enable our proposed approach.

Given a reference model $\piref$, we first define the following modified reward function
\begin{align}\neurips{\textstyle}
\label{def:ar_rwd}
R_{\beta}(s_h,a_h) = \frac{r(s_h,a_h)}{\beta} + \log\piref(a_h\mid s_h),
\end{align}
and, therefore, the original KL-regularized RL objective can be rewritten as:\neurips{ $J_{\beta}(\pi) = \beta \cdot \E_{\tau \sim \pi} \left[ R_{\beta}(\tau) - \log\pi(\tau) \right] = \beta \cdot \E_{\tau \sim \pi} \left[ \sum_{h=1}^H \left( R_{\beta}(s_h,a_h) - \log\pi(a_h\mid s_h) \right) \right]$.}
\arxiv{\begin{align*}
J_{\beta}(\pi) = \beta \cdot \E_{\tau \sim \pi} \left[ R_{\beta}(\tau) - \log\pi(\tau) \right] = \beta \cdot \E_{\tau \sim \pi} \left[ \sum_{h=1}^H \left( R_{\beta}(s_h,a_h) - \log\pi(a_h\mid s_h) \right) \right].
\end{align*}}
Here, KL-regularization is equivalent to entropy regularization. However, we will continue to use the KL-regularization terminology throughout the remainder of the paper for consistency.

The optimal policy for the above objective, denoted $\pi^\star_{\beta} = \text{arg}\max_{\pi} J_{\beta}(\pi)$, has a closed-form solution that takes the form of a softmax distribution,
\begin{align*}
\pi^\star_{\beta}(a_h\mid s_h) \propto \piref(a_h\mid s_h) \exp\left(\nicefrac{Q^\star_{r}(s_h,a_h)}{\beta}\right), \quad \text{or} \quad \pi^\star_{\beta}(a_h\mid s_h) \propto \exp\left(Q^\star_{R_\beta}(s_h,a_h)\right),
\end{align*}
where $Q^\star_{r}$ and $Q^\star_{R_\beta}$ are the optimal action-value functions for the original reward $r$ and transformed reward functions $R_\beta$, respectively.

We now formalize the Bellman operators under the shifted reward function $R_\beta$, and discuss the key properties induced by the KL-regularization. For a given policy $\pi$ and any Q-function $Q$, we define the Bellman operator as
\neurips{\begin{small}}\begin{align*}
(\Tcal^\pi_\beta Q)(s_h,a_h) \coloneqq &~ R_\beta(s_h,a_h) + \E_{s_{h+1} \sim \Pcal(\cdot \mid s_h,a_h), a_{h+1} \sim \pi(\cdot \mid s_{h+1})} \left[ Q(s_{h+1},a_{h+1}) - \log\pi(a_{h+1} \mid s_{h+1}) \right].
\end{align*}\neurips{\end{small}}
Throughout this paper, we will use $V_Q(s) \coloneqq \log \sum_{a \in \Acal} e^{Q(s,a)}$ to denote the softmax of the given $Q$. With this, the Bellman optimality operator becomes
\neurips{\begin{small}}\begin{align*}
(\Tcal_\beta Q)(s_h,a_h) \coloneqq &~ R_\beta(s_h,a_h) + \E_{s_{h+1} \sim \Pcal(\cdot \mid s_h,a_h)} \Big[\underbrace{\max_{\pi\in\Delta(\Acal)} \E_{a_{h+1} \sim \pi(\cdot \mid s_{h+1})}\left[ Q(s_{h+1},a_{h+1}) - \log\pi(a_{h+1} \mid s_{h+1}) \right]}_{=\log \sum_{a \in \Acal} \exp(Q(s_{h+1},a))\eqqcolon V_Q(s_{h+1})}\Big] \\
= &~ R_\beta(s_h,a_h) + \E_{s_{h+1} \sim \Pcal(\cdot \mid s_h,a_h)} \left[ V_Q(s_{h+1}) \right].
\end{align*}\neurips{\end{small}}
With the above definitions, the optimal policy takes the form of the following Boltzmann distribution:
\neurips{\begin{small}}\begin{align*}
\pi_Q(\cdot \mid s_h) \coloneqq \argmax_{\pi \in \Delta(\Acal)} \E_{a_{h} \sim \pi(\cdot \mid s_{h})}\left[ Q(s_{h},a_{h}) - \log \pi(a_{h} \mid s_{h}) \right] = e^{Q(s_h,\cdot) - V_Q(s_h)}.
\end{align*}\neurips{\end{small}}
We can further define the KL-regularized value functions for a given policy $\pi$:
\neurips{\begin{small}}\begin{gather*}
\neurips{\textstyle}
Q_\beta^\pi(\widetilde s_h,\widetilde a_h) \coloneqq \E_{\pi} \left[ R_\beta(s_{h},a_{h}) + \sum_{h'=h+1}^H \left( R_\beta(s_{h'},a_{h'}) - \log\pi(a_{h'} \mid s_{h'}) \right) \midmid (s_h,a_h) = (\widetilde s_h,\widetilde a_h) \right], \\
\neurips{\textstyle}
V_\beta^\pi(\widetilde s_h) \coloneqq \E_{\pi} \left[ \sum_{h'=h}^H \left( R_\beta(s_{h'},a_{h'}) - \log\pi(a_{h'} \mid s_{h'}) \right) \midmid s_h = \widetilde s_h \right] = \E_{a_h \sim \pi(\cdot \mid \widetilde s_h)} \left[ Q_\beta^\pi(\widetilde s_h, a_h) - \log\pi(a_h \mid \widetilde s_h) \right].
\end{gather*}\neurips{\end{small}}
The corresponding optimal value functions are $Q_\beta^\star \coloneqq Q^{\pi^\star_\beta}_\beta$ and $V_\beta^\star \coloneqq V^{\pi^\star_\beta}_\beta$. Here, $Q_\beta^\pi$ and $Q_\beta^\star$ are also the fixed points of $\mathcal{T}_\beta^\pi$ and $\mathcal{T}_\beta$, respectively.

For autoregressive function approximation architectures, such as large language models, we can directly leverage logits  to parameterize $Q$, $V_Q$, and $\pi_Q$ as follows.
Let $\theta$ be the model weights. Define
\begin{gather}
\label{eq:ar_func_approx}
\boxed{Q_\theta(s,a) \coloneqq \logit_\theta(s,a),\,
V_\theta(s) \coloneqq \softmax\circ\logit_\theta(s,\cdot),\,
\log\pi_\theta(a\mid s) \coloneqq Q_\theta(s,a) - V_\theta(s),}
\end{gather}
where $\softmax\circ\logit_\theta(s,\cdot) \coloneqq \log\sum_{a \in \Acal} \exp(\logit_\theta(s,a))$.\footnote{Here, $\softmax$ denotes the $\mathsf{log\text{-}sum\text{-}exp}$ operator for notational convenience, rather than the vocabulary softmax layer used to produce token probabilities in language models.} While we assumed here temperature to be $1$ for simplicity, any temperature can be incorporated by appropriately scaling the reward \cref{def:ar_rwd}.

\section{Trajectory Bellman Residual Minimization}\label{sec:tbrm}

In this section, we introduce our main algorithm, \textbf{Trajectory Bellman Residual Minimization (TBRM)}, designed specifically for large language models problems.
As we discussed in \cref{sec:kl_rl}, the transition dynamics of large language models can be viewed as deterministic.
\emph{For the remainder of this section, we will apply the autoregressive function approximation defined in \cref{eq:ar_func_approx} and assume deterministic transition dynamics.}

Recall that Bellman error $Q(s_h,a_h) - (\Tcal_\beta Q)(s_h,a_h)$ over state-action pairs $(s_h,a_h)$ is employed as the proxy for controlling the performance $J_\beta(\pi_Q)$ of $\pi_Q$. Minimizing the square of Bellman error on $(s_h,a_h)$ in deterministic MDPs is equivalent for minimizing the square of Bellman residual $Q(s_h,a_h) - R_\beta(s_h,a_h) - V_Q(s_{h+1})$ given the $(s_h,a_h,s_{h+1})$ tuple. This leads to the classical Bellman residual minimization objective \citep[BRM;][]{schweitzer1985generalized,baird1995residual}, which we expand using the definition of $R_\beta$ and autoregressive function approximation in \cref{eq:ar_func_approx}:
\begin{gather*}
\label{eq:classical_brm}\numberthis
\Lcal_{\Dcal}^{\sf BRM}(\theta) = \frac{1}{|\Dcal|}\sum_{\tau\in\Dcal} \color{blue}\sum_{h=1}^{H}\color{black} \Big(\logit_\theta(s_h,a_h) - (\Tcal_\beta \logit_\theta)(s_h,a_h)\Big)^2
\\
= \frac{1}{|\Dcal|}\sum_{\tau\in\Dcal} \color{blue}\sum_{h=1}^{H}\color{black} \left(\logit_\theta(s_h,a_h) - \frac{r(s_h,a_h)}{\beta} - \log\piref(s_h,a_h) - \log\sum_{a\in\calA}\exp(\logit_\theta(s_{h+1},a))\right)^2.
\end{gather*}
Here $\Dcal$ denotes data which can be either purely offline or updated online as a replay buffer. In the context of LLMs, directly minimizing BRM may not be possible because the token-level reward signal is either unavailable (e.g., if we assign the outcome reward to the final token) or very sparse.
Crucially, minimizing the square of per-step Bellman error as in $\Lcal_{\Dcal}^{\sf BRM}(\theta)$ is sufficient but not necessary for maximization of $J_\beta(\pi_\theta)$. Indeed,  a weaker control of Bellman errors over certain distributions is sufficient for optimizing $J_\beta(\pi_\theta)$ (see, e.g., \citet{xie2020q} or \cref{cor:all-pol-sub-opt}). As we prove below (\cref{sec:tbrm_theory}), it is sufficient to instead consider a trajectory-based variant of BRM,
\begin{gather*}
\Lcal_{\Dcal}^{\sf TBRM}(\theta) = \frac{1}{|\Dcal|}\sum_{\tau\in\Dcal} \left( {\color{blue}\sum_{h=1}^{H}} \logit_\theta(s_h,a_h) - (\Tcal_\beta \logit_\theta)(s_h,a_h)\right)^2
\\
\label{eq:brm_mc_obj}\numberthis
= \frac{1}{|\Dcal|}\sum_{\tau\in\Dcal} \left(\logit_\theta(s_1,a_1) - \frac{r(\tau)}{\beta} - \log\piref(\tau) + \sum_{h=2}^{H}\log\pi_\theta(a_{h} \mid s_{h}) \right)^2,
\end{gather*}
where the second equality used the autoregressive function approximation, i.e., $\log\pi_\theta(a_{h} \mid s_{h}) = \logit_\theta(a_{h} \mid s_{h}) - \log\sum_{a\in\calA}\exp(\logit_\theta(s_{h},a))$, to simplify the expression. An immediate benefit of TBRM is that we only require the trajectory outcome $r(\tau)$, rather than the process reward $r(s_h,a_h)$ as in BRM, and credit assignment will be (provably) carried out in the learning procedure itself.
\arxiv{We now formally let $\Dcal \coloneq \{\tau^\iter{i}\}_{i=1}^{|\Dcal|}$be a trajectory dataset with  outcome rewards, $\tau^\iter{i} = (s_1^\iter{i},a_1^\iter{i},\dotsc,s_H^\iter{i},a_H^\iter{i},r(\tau^\iter{i}))$.}
In our experiments, we update $\Dcal$ online, and \cref{alg:brm_mc} represents the exact implementation we used in \cref{sec:exp}.

\begin{algorithm}[ht]
    \caption{Trajectory Bellman Residual Minimization (TBRM)}
    \label{alg:brm_mc}
    {\bfseries input:} Task prompt dataset $\Dcal_\mathsf{task}$, reward function $r$, reward scaling coefficient $\beta$, reference policy $\piref$ with parameter $\theta_\rf$, and number of iterations $T$.
    \begin{algorithmic}[1]
        \State Initialize $\theta \leftarrow \theta_\rf$.
        \For{$t = 1$ to $T$}
        \State Sample a batch $\mathcal{D}_\mathsf{task}^t\subset \Dcal_\mathsf{task}$.
        \State For each question $q\in \mathcal{D}_\mathsf{task}^t$, sample a trajectory $\tau$ from policy $\pi_\theta$ with initial state $s_1 = q$, and collect these trajectories into dataset $\Dcal_t$.
        \State Update $\theta$ via gradient descent to minimize $\Lcal_{\Dcal_t}^{\sf TBRM}(\theta)$ as defined in \cref{eq:brm_mc_obj}:
            \begin{small}
            \begin{align*}
            \nabla_\theta \frac{1}{|\Dcal_t|}\sum_{\tau\in\Dcal_t} \left(\logit_\theta(s_1,a_1) - \frac{r(\tau)}{\beta} - \log\piref(\tau) + \sum_{h=2}^{H}\log\pi_\theta(a_{h} \mid s_{h}) \right)^2,
            \end{align*}
            \end{small}
            where $\logit_\theta(s_1,a_1)$ is the raw logit of the first output token. 
        \EndFor
        \State \Return $\pi_\theta$.
        \end{algorithmic}
    \end{algorithm}
\neurips{\vspace{-2pt}}

As we formally prove in \cref{sec:tbrm_theory}, TBRM is a fully \emph{off-policy} algorithm with a directly optimizable objective, and it provably converges to a near-optimal policy with any off-policy data (though the degree of off-policyness may affect sample efficiency, \cref{eq:conc}). In contrast, policy-based counterparts are usually on-policy in nature: policy-gradient-based algorithms (like REINFORCE) require to sample new trajectories in an on-policy manner. PPO-like algorithms  require on-policy actions for their actor components but optimize a surrogate loss instead. PPO's critic update also requires on-policy rollouts.
In contrast, the off-policy nature of TBRM removes the need for additional techniques such as importance sampling ratios, clipping, critic models, or (multiple) on-policy rollouts.

Readers may question: \emph{why hasn't this simple variant of BRM received attention in the literature?} We conjecture that TBRM has been suspected to suffer from the curse of horizon, at least from the theoretical perspective. In the theory of offline RL \citep[e.g.,][]{chen2019information,xie2020q}, a key technique is to control the expected Bellman error $|\E_{\pi}[\sum_h ((\Tcal_\beta Q)(s_h,a_h) - Q(s_h,a_h)))]|$ on a certain (unavailable) distribution $d^\pi$ by instead minimizing the per-step squared Bellman error $\sum_h\E_{\mu}[((\Tcal_\beta Q)(s_h,a_h) - Q(s_h,a_h))^2]$ on the data distribution $d^\mu$ generated by $\mu$. This step only incurs the cost of the state-wise distribution-shift  $\frac{d^\pi(s_h,a_h)}{d^\mu(s_h,a_h)}$. When it comes to trajectory-level data, minimization of the square of expected Bellman error $\E_{\mu}[(\sum_h[(\Tcal_\beta Q)(s_h,a_h) - Q(s_h,a_h)])^2]$, as in TBRM, would appear to incur the trajectory-level distribution-shift cost $\Pi_h \frac{\pi(a_h \mid s_h)}{\mu(a_h \mid s_h)}$, and thus possibly cause an exponential blow-up with horizon $H$ compared to the state-wise case. However, the recent theoretical results \citep{jia2025we} challenge this conventional wisdom and indicate that the Markov property can be the key to avoiding trajectory-level distribution-shift when conducting trajectory-level change of measure.
In \cref{sec:tbrm_theory}, we formally prove that TBRM indeed only incurs state-wise distribution-shift regardless of its trajectory-level objective, and show that TBRM can efficiently converge to a near-optimal policy with finite-sample analysis.

\paragraph{Comparison with other related algorithms}
We note that algorithms with similar structure to TBRM have been derived previously from diverse perspectives in both deep RL \citep{haarnoja2017reinforcement,schulman2017equivalence,nachum2017bridging,haarnoja2018soft} and LLM applications \citep{guo2022efficient,ethayarajh2024kto,team2025kimi,ji2024enhancing,wang2024offline}. This convergence is unsurprising, as TBRM and related algorithms fundamentally aim to minimize Bellman error, albeit through different formulations and optimization approaches.
However, to the best of our knowledge, TBRM is the only optimization algorithm (rather than iterative ones like Q-learning; \cref{app:hard-instances} demonstrates the benefit of optimization over iteration) that requires only one rollout per prompt among all of these approaches. \cref{sec:comparison} provides a detailed comparison of TBRM with other related algorithms.
The present paper formally establishes finite-sample guarantees for TBRM.

\subsection{Theoretical Analysis of TBRM}
\label{sec:tbrm_theory}

\newcommand{\Qs}{Q^\star}
\newcommand{\Vs}{V^\star}
\newcommand{\ths}{{\theta^\star}}

We use $\Theta$ to denote the parameter space, equipped with norm $\nrm{\cdot}$. We assume the following standard realizability condition, which can be relaxed to hold approximately \citep[see, e.g.,][]{cheng2022adversarially}.

\begin{assumption}[Realizability]\label[assumption]{asmp:realizable}
There exists $\ths\in\Theta$ such that $Q_{\ths}=Q^\star$.
\end{assumption}

\paragraph{Motivation}
We first show that $\ths$ is the population minimizer of the TBRM loss~\eqref{eq:brm_mc_obj} through Bellman equation. Under the parametrization \eqref{eq:ar_func_approx}, $Q_{\ths}(s,a)=\logit_{\ths}(s,a)$ is the optimal soft Q-function for the transformed reward function $R_\beta(s,a)=\frac{r(s,a)}{\beta}+\log\piref(a\mid s)$, and the optimal value function is given by $V_{\ths}(s)=Q_{\ths}(s,a)-\log \pi_{\ths}(a\mid s)$. Therefore, 
the Bellman equation becomes (deterministically for a trajectory $\tau$ drawn from the MDP)
\begin{align*}
    Q_{\ths}(s_h,a_h)=R_\beta(s_h,a_h)+V_{\ths}(s_{h+1})= R_\beta(s_h,a_h)+Q_{\ths}(s_{h+1},a_{h+1})-\log \pi_{\ths}(a_{h+1}\mid s_{h+1}).
\end{align*}
Then, summing over $h=1,2,\cdots,H-1$ for any admissible trajectory $\tau$, we have
\begin{align*}
0\equiv&~ \sum_{h=1}^{H-1} \brac{ Q_{\ths}(s_h,a_h)-R_\beta(s_h,a_h)-Q_{\ths}(s_{h+1},a_{h+1})+\log \pi_{\ths}(a_{h+1}\mid s_{h+1}) } \\
=&~\logit_{\ths}(s_1,a_1) - R(\tau) + \sum_{h=2}^{H}\log\pi_{\ths}(a_{h} \mid s_{h}),
\end{align*}
where $R(\tau)=\frac{r(\tau)}{\beta}+\log\piref(\tau)$ is the trajectory transformed reward. Hence, it holds that $\Lcal_{\Dcal}^{\sf TBRM}(\ths)\equiv 0$ deterministically, and any approximate minimizer of the loss $\Lcal_{\Dcal}^{\sf TBRM}$ must also attain low trajectory Bellman residual.

The analysis above establishes a necessary condition for the optimal soft Q-function $Q_{\ths}(s,a)$ or $\logit_{\ths}(s,a)$.
Beyond this, the sub-optimality of a policy $\pi_\theta$ can also be related to the trajectory Bellman residual: through our analysis in \cref{app: perf-diff}, we can show that for any $\theta\in\Theta$,
\begin{align*}
J_\beta(\pi^\star) - J_\beta(\pi_\theta)
\leq &~ 2 \beta \max_{\pi\in\set{\pi^\star,\pi_\theta}} \left| \E_{\pi}\left[\logit_{\theta}(s_1,a_1) - R(\tau) + \sum_{h=2}^{H}\log\pi_{\theta}(a_{h} \mid s_{h})\right] \right|.
\end{align*}
Therefore, it remains to relate the expected trajectory Bellman residual under the off-policy distribution induced by $\mu$ and any policy $\pi$ through change-of-trajectory-measure.

\paragraph{Change-of-trajectory-measure}
A key to our analysis is the following improved version of the \emph{change-of-trajectory-measure lemma} \citep[]{jia2025we}.
Let $\chisq{P}{Q}=\EE_P[d P/dQ] - 1$ be the $\chi^2$-divergence.
Let $d_h^\pi(\cdot), d_h^{\mu}(\cdot)\in \Delta(\calS_h\times\calA)$  denote the occupancy measures of Markovian policies $\pi$ and $\mu$.
\begin{lemma}[Change-of-Trajectory-Measure Lemma]\label[lemma]{lem: change-of-trajectory-measure}
    Given an MDP $M = (H, \calS, \calA, \calP, r, \rho)$ and a policy $\pi$,
    \begin{align}\neurips{\textstyle}
        \label{eq:cotm-improved}
        \sup_{f}~ \frac{\left(\EE_{\pi}\left[\sum_{h=1}^H f(s_h, a_h)\right]\right)^2}{\EE_{\mu}\left[\left(\sum_{h=1}^H f(s_h, a_h)\right)^2\right]}\le 1+\sum_{h=1}^H \chisq{d_h^\pi}{d_h^{\mu}},
    \end{align}
    where the supremum is over all measurable functions $f: \calS\times\calA\to \RR$.
\end{lemma}
The proof of \cref{lem: change-of-trajectory-measure} significantly simplifies the one in \citet{jia2025we}, as shown in \cref{appendix:cotm_proof}.
As a direct corollary, the RHS of \cref{eq:cotm-improved} can be further upper bounded by $H\cdot \max_{h,s,a}\frac{d_h^\pi(s_h, a_h)}{d_h^{\mu}(s_h, a_h)}$, improving upon  \citet{jia2025we} by a factor of $H^2$.

Our main result will be stated in terms of the following \emph{concentrability coefficient} of the data collection policy $\mu$:
\begin{align}\label{eq:conc}
    \Conchis{\mu}\defeq 1+\max_{\theta\in\Theta}\max_{h\in[H]}\chisq{d^{\pi_{\theta}}_h}{d^\mu_h},
\end{align}
a notion weaker than the commonly-studied  $L_\infty$-concentrability $\Concinf{\mu}\defeq \max_{\pi,h,s,a} \frac{d^\pi_h(s,a)}{d^\mu_h(s,a)}$. %

The proof of the following theorem is deferred to \cref{app: proof-tbrm}.

\begin{theorem}\label{thm:tbrm_main}
Suppose $\htheta$ is a parameter that satisfies
$\Lcal_{\Dcal}^{\sf TBRM}(\htheta)-\inf_{\theta\in\Theta}\Lcal_{\Dcal}^{\sf TBRM}(\theta)\leq \epsopt$, data $\Dcal$ are i.i.d. according to $\mu$, and \cref{asmp:theta-bound,asmp:parametric-theta,asmp:realizable} hold. Then, with high probability, it holds that
\begin{align*}%
    J_{\beta}(\pi^\star) - J_{\beta}(\pi_{\htheta}) \leq \tbO\paren{ \sqrt{H\Conchis{\mu} \paren{\beta^2\epsopt+\frac{H^2\dim(\Theta)}{|\cD|}}} },
\end{align*}
where $\dim(\Theta)$ is the measure of the dimension of $\Theta$ defined in \cref{asmp:parametric-theta}. 
\end{theorem}

\section{Experiments}\label{sec:exp}
In this section, we present experiments to evaluate the performance of TBRM on reasoning tasks. We compare TBRM against two policy-based methods: GRPO and PPO. The codebase for the experiments is publicly available at \url{https://github.com/rlx-lab/TBRM}.

\subsection{Experiment Setup}\label{sec:exp-setup}

\paragraph{Datasets and models}
We train our models using the prompt set from DAPO \citep[][Apache license 2.0]{yu2025dapo}, which comprises approximately 17.4k math problems sourced from the AoPS\footnote{\url{https://artofproblemsolving.com/}} website and official competition homepages. All problems are standardized to have integer answers. To demonstrate the generality of our method across model scales, we conduct experiments using Qwen2.5-Math-1.5B and Qwen2.5-Math-7B.

\paragraph{Evaluation}
We assess the models' reasoning abilities on several standard math benchmarks: AIME24, AIME25, AMC23, MATH500 \citep{hendrycks2021measuring}, Minerva Math \citep{lewkowycz2022solving}, and OlympiadBench \citep{he2024olympiadbench}. For MATH500, Minerva Math, and OlympiadBench, we generate a single response per problem and report the overall accuracy, denoted as \textbf{Avg@1}. For the smaller benchmarks AIME24, AIME25, and AMC23, where performance can fluctuate due to limited data, we generate 32 responses per problem and average the accuracies to mitigate the intrinsic randomness of LLM outputs; this metric is denoted as \textbf{Avg@32}. Responses are sampled with temperature 0 for \textbf{Avg@1} and temperature 1.0 for \textbf{Avg@32}. We employ Math-Verify \citep[][Apache-2.0 license]{Kydlicek_Math-Verify_Math_Verification} as the verifier.

\paragraph{Implementation details}
We implement our methods and baselines using the VERL framework \citep[][Apache-2.0 license]{sheng2024hybridflow}, following most of the recommended hyperparameter settings for GRPO and PPO. To balance performance and efficiency, we use a prompt batch size of 128 and a response length of 2048 tokens per training step. For PPO, we generate one response per prompt ($n = 1$), while for GRPO, which requires multiple rollouts, we generate four responses ($n = 4$). On TBRM, we experiment with both settings ($n = 1$ and $n = 4$). All responses are sampled with a temperature of 1.0. For TBRM, we set $\beta = 0.002$ across all experiments. All models are trained for the same number of steps. More details of our implementation can be found in \cref{appendix:impl}. 

\subsection{Main Results}\label{sec:exp-main-result}
The effectiveness of our algorithm is demonstrated in \cref{tab:results}. Across six challenging math benchmarks, TBRM consistently matches or surpasses its comparable baselines. Specifically, with a single rollout per prompt, TBRM$_{n=1}$ achieves higher accuracies than PPO$_{n=1}$ on most benchmarks and matches the performance of GRPO$_{n=4}$, despite the latter using four times as many samples during training. Notably, our Qwen2.5-Math-1.5B-based model attains 13.2\% accuracy on AIME24, outperforming both the Qwen2.5-Math-7B base model and the 1.5B GRPO$_{n=4}$ model. When increasing the number of sampled responses to four, TBRM$_{n=4}$ surpasses baselines by a larger margin. On AIME24, our 1.5B model reaches 14.3\% accuracy, while the 7B model further advances to 30.5\%, exceeding GRPO$_{n=4}$ by 1.6\%. Additional results with more rollouts are presented in \cref{app:more-rollouts}.

\renewcommand{\arraystretch}{1.2}
\newcommand{\auxinfofont}[1]{{\scriptsize \textbf{#1}}}
\definecolor{bgcolor}{HTML}{D0E0F7}
\begin{table*}[ht!]
    \begin{adjustbox}{width=\textwidth,center}
    \centering
    \begin{tabular}{lcccccc}
        \toprule
        \textbf{Method} & \textbf{AIME24} & \textbf{AIME25} & \textbf{AMC23} & \textbf{MATH500} & \textbf{Minerva Math} & \textbf{OlympiadBench} \\
        & \auxinfofont{Avg@32} & \auxinfofont{Avg@32} & \auxinfofont{Avg@32} & \auxinfofont{Avg@1} & \auxinfofont{Avg@1} & \auxinfofont{Avg@1} \\
        \midrule
        \textbf{Qwen2.5-Math-1.5B} & 5.0 & 1.9 & 24.9 & 63.4 & 16.5 & 30.8 \\
        \makebox[3em][l]{PPO} \auxinfofont{$n=1$} & 11.4 & 4.5 & 46.6 & \underline{\textbf{72.2}} & 26.8 & \underline{36.0} \\
        \rowcolor{bgcolor} \makebox[3em][l]{TBRM} \auxinfofont{$n=1$} & \underline{13.2} & \underline{5.6} & \underline{48.6} & \underline{\textbf{72.2}} & \underline{27.2} & 35.7 \\
        \makebox[3em][l]{GRPO} \auxinfofont{$n=4$} & 13.0 & \textbf{7.1} & 49.9 & 71.2 & 28.7 & \textbf{37.5} \\
        \rowcolor{bgcolor} \makebox[3em][l]{TBRM} \auxinfofont{$n=4$} & \textbf{14.3} & 6.9 & \textbf{52.0} & \textbf{72.2} & \textbf{30.5} & 36.1 \\
        \midrule
        \textbf{Qwen2.5-Math-7B} & 10.6 & 2.8 & 31.6 & 67.4 & 13.2 & 29.3 \\
        \makebox[3em][l]{PPO} \auxinfofont{$n=1$} & \underline{25.4} & \underline{\textbf{13.2}} & \underline{63.4} & 76.4 & 33.8 & 39.3 \\
        \rowcolor{bgcolor} \makebox[3em][l]{TBRM} \auxinfofont{$n=1$} & 24.1 & \underline{\textbf{13.2}} & \underline{63.4} & \underline{78.6} & \underline{\textbf{36.4}} & \underline{41.5} \\
        \makebox[3em][l]{GRPO} \auxinfofont{$n=4$} & 28.9 & 10.7 & 66.8 & \textbf{79.8} & 36.0 & 42.5 \\
        \rowcolor{bgcolor} \makebox[3em][l]{TBRM} \auxinfofont{$n=4$} & \textbf{30.5} & 13.1 & \textbf{68.4} & \textbf{79.8} & \textbf{36.4} & \textbf{44.1} \\
        \bottomrule
    \end{tabular}
    \end{adjustbox}
    \caption{Performance of various methods on math benchmarks, where $n$ denotes the number of responses sampled per prompt during training. For each benchmark, the highest accuracy across all methods is bolded, and the highest accuracy among methods with $n = 1$ is underscored.}
    \label{tab:results}
\end{table*}

\begin{figure}[th]
    \centering
    \begin{subfigure}[t]{0.655\textwidth}
        \includegraphics[width=1.0\linewidth]{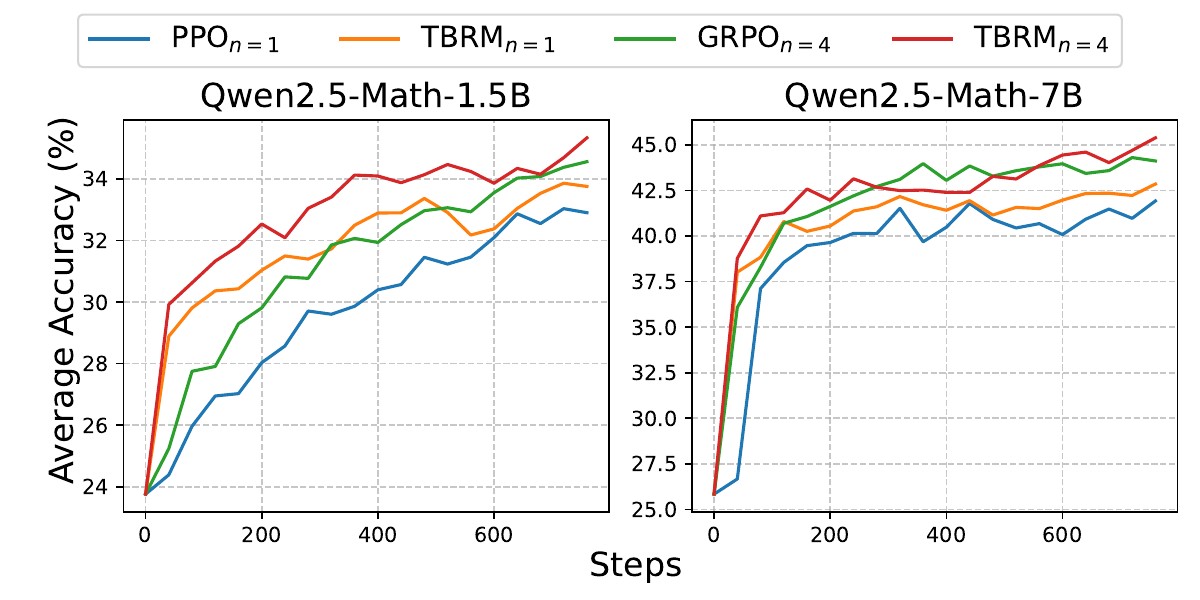}
        \caption{Average learning accuracy across benchmarks}
        \label{fig:average_accuracy}
    \end{subfigure}
    \begin{subfigure}[t]{0.325\textwidth}
        \includegraphics[width=1.0\linewidth]{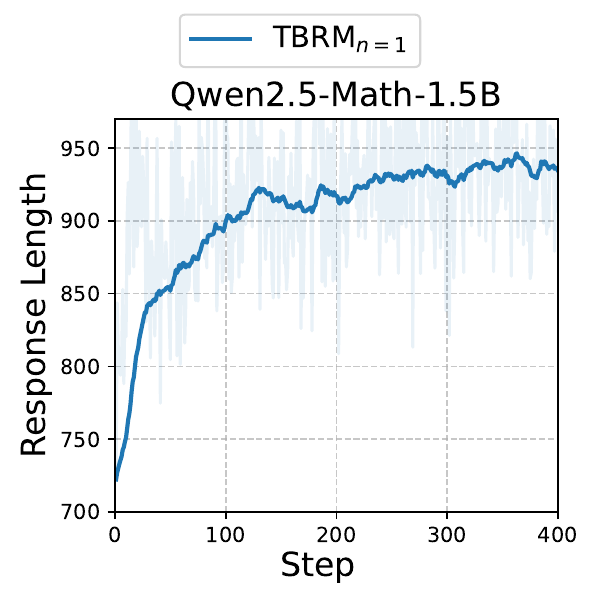}
        \caption{Train-time response length. }
        \label{fig:response_length}
    \end{subfigure}
    \caption{\cref{fig:average_accuracy} shows average learning accuracy across benchmarks for PPO, GRPO, and TBRM. Per-benchmark results can be found in \cref{appendix:perbench_accuracy}. \cref{fig:response_length} demonstrates model's ability to engage in extended reasoning at test time with R1 template.
    }
\end{figure}

\cref{fig:average_accuracy} illustrates the step-wise average performance across all benchmarks for TBRM and baseline methods. While all approaches demonstrate improved reasoning with increased training data, TBRM consistently exhibits a superior convergence rate and achieves higher absolute performance than its counterparts. Notably, TBRM$_{n=4}$ attains the highest performance throughout nearly the entire training duration. Futhermore, TBRM$_{n=1}$ outperforms PPO$_{n=1}$ and performs similarly to GRPO$_{n=4}$, especially on 1.5B model, with only a mild gap in their results.

\subsection{Training Dynamics and Performance Analysis}

\paragraph{Reward}
We present the training reward curves in \cref{fig:reward}, which shows that TBRM achieves comparable reward levels to its baselines. Furthermore, TBRM demonstrates a faster convergence rate during early training. This is particularly evident with the 1.5B model, where TBRM attains significantly higher rewards than PPO and GRPO. 

\begin{figure}[th]
    \centering
    \begin{subfigure}[t]{0.655\textwidth}
        \includegraphics[width=1.0\linewidth]{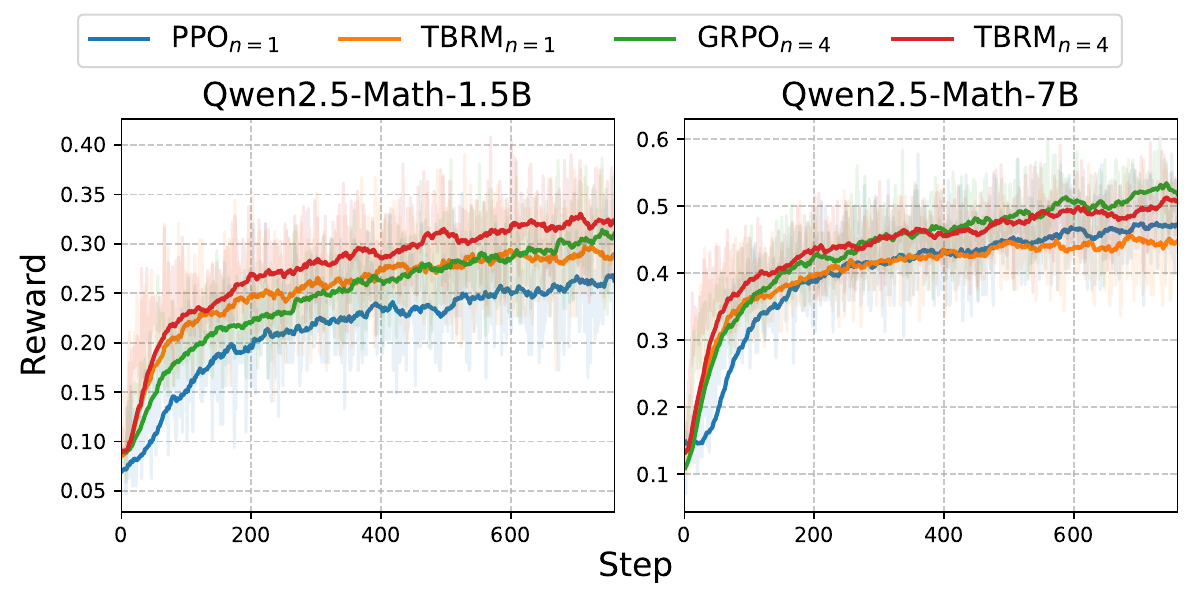}
        \caption{Training reward curves. }
        \label{fig:reward}
    \end{subfigure}
    \begin{subfigure}[t]{0.325\textwidth}
        \includegraphics[width=1.0\linewidth]{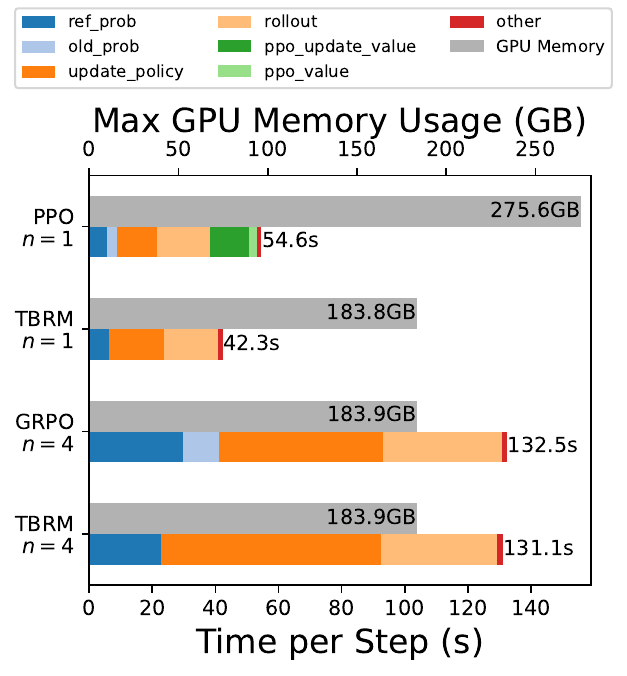}
        \caption{Train-time overhead. }
        \label{fig:train_eff}
    \end{subfigure}
    \caption{\cref{fig:reward} demonstrates the increment of rewards during training. \cref{fig:train_eff} is a comparison of maximal GPU memory consumption and per-step time cost across different methods with base model Qwen2.5-Math-7B. Time cost is segmented into key partitions, with time partition labels defined in \cref{tab:train-eff-label-desc} (\cref{appendix:train-eff}).
    }
\end{figure}

\paragraph{Response length}\label{par:response_length}
Prior work has shown that reinforcement learning can enhance a model's ability to solve increasingly complex reasoning tasks by leveraging extended test-time computation, as reflected in progressively longer responses during training \citep{guo2025deepseek,zeng2025simplerl,liu2025understanding}. We find that TBRM exhibits a similar capability. Specifically, we adopt the prompt template from DeepSeek-R1 \citep{guo2025deepseek} and apply TBRM training to Qwen2.5-Math-1.5B. Following previous studies \citep{zeng2025simplerl}, we include only responses that terminate under normal conditions---excluding those truncated due to length limits---as truncated outputs often suffer from repetition and incompleteness. As illustrated in \cref{fig:response_length}, TBRM encourages the model to explore and refine its reasoning more deeply over time, enabling models to take advantage of extended computation at test time to improve reasoning performance. 

\paragraph{Training efficiency}\label{par:train-eff}
TBRM offers substantial implementation simplicity relative to existing approaches. Specifically, it eliminates the need for critic models $V_\phi$, as required by PPO, and avoids the necessity of sampling multiple responses per prompt ($n > 1$), as in GRPO. Moreover, due to its fully off-policy nature, TBRM does not require multiple updates per training step. We conduct experiments using Qwen2.5-Math-7B under consistent training conditions (see \cref{appendix:train-eff} for details) and report the peak GPU memory usage and wall-clock time per training step for each method in \cref{fig:train_eff}. Overall, TBRM exhibits matched or lower resource consumption compared to its counterparts. When $n = 1$, TBRM$_{n=1}$ uses $33.3\%$ less GPU memory than PPO$_{n=1}$ and achieves a $1.3\times$ speedup. For $n = 4$, TBRM$_{n=4}$ demonstrates comparable resource usage to GRPO$_{n=4}$. Notably, TBRM$_{n=1}$ achieves a $3.1\times$ training speedup relative to GRPO$_{n=4}$, while yielding similar performance despite sampling only a single response---GRPO$_{n=4}$ outperforms TBRM$_{n=1}$ by only $0.80\%$ with the 1.5B model and $1.27\%$ with the 7B model on average across math benchmarks.

\paragraph{Extended Analysis}
We further analyze the responses generated by the TBRM models and identify several notable reasoning patterns, including verification, backtracking, decomposition, and enumeration. Illustrative examples of these patterns are provided in \cref{appendix:qual-analysis}. In \cref{app:classical-brm}, we compare TBRM with the classical, token-level BRM formulation and show that directly applying BRM to LLMs leads to unstable training and reward collapse, underscoring the importance of the trajectory-level design. In \cref{app:reasoning-gym}, we evaluate TBRM on a suite of non-mathematical reasoning tasks, demonstrating its ability to generalize beyond the mathematical domain.

\arxiv{\input{sections/related-work}}
\arxiv{\input{sections/comparison}}

\section{Conclusion}\label{sec:conclusion}

In this paper, we present TBRM, a simple yet effective value-based RL algorithm for enhancing LLM reasoning capabilities. TBRM operates efficiently with just a single rollout per prompt and employs a lightweight optimization objective, eliminating the need for critics, importance-sampling ratios, or clipping mechanisms that are commonly required in policy-based approaches. Our theoretical analysis demonstrates that TBRM is guaranteed to converge to a near-optimal policy using off-policy data, while our empirical results show its effectiveness across standard mathematical reasoning benchmarks. We hope this work may contribute to expanding interest in value-based approaches for LLM reasoning, potentially complementing the policy-based algorithms that currently dominate LLM post-training methods.

\section*{Acknowledgements}
We acknowledge support of the Simons Foundation and the NSF through awards DMS-2031883 and PHY-2019786,  ARO through award W911NF-21-1-0328, and the DARPA AIQ award.

\neurips{
\clearpage}
\bibliographystyle{plainnat}
\bibliography{ref}

\clearpage

\appendix
\onecolumn

\begin{center}
{\LARGE Appendix}
\end{center}

\renewcommand{\contentsname}{}
\addtocontents{toc}{\protect\setcounter{tocdepth}{2}}
{\hypersetup{hidelinks}
\tableofcontents
}

\neurips{\input{sections/related-work}}
\input{sections/proof-appendix}
\neurips{\input{sections/comparison}}
\input{sections/hard-inst}

\input{sections/exp-appendix}

\end{document}

%% file: macro.tex
\newcommand{\calA}{\mathcal{A}}

\newcommand{\calP}{\mathcal{P}}

\newcommand{\calS}{\mathcal{S}}

\renewcommand{\epsilon}{\varepsilon}

\newcommand{\indep}{\perp\kern-6pt \perp}

\newcommand{\barf}{\bar{f}}

\newcommand{\chisq}[2]{\chi^2\left(\left.#1\ \right\|\ #2\right)}
\newcommand{\cond}[2]{\left[\left.#1\ \right|\ #2\right]}

\newcommand{\htheta}{\widehat{\theta}}
\newcommand{\epsopt}{\varepsilon_{\mathsf{opt}}}

%% file: command.tex
\usepackage{makecell}

\usepackage{prettyref}
\newcommand{\pref}[1]{\prettyref{#1}}

\newcommand{\savehyperref}[2]{\texorpdfstring{\hyperref[#1]{#2}}{#2}}
\newrefformat{eq}{\savehyperref{#1}{Eq.~\textup{(\ref*{#1})}}}
\newrefformat{eqn}{\savehyperref{#1}{Equation~\ref*{#1}}}
\newrefformat{lem}{\savehyperref{#1}{Lemma~\ref*{#1}}}
\newrefformat{lemma}{\savehyperref{#1}{Lemma~\ref*{#1}}}
\newrefformat{def}{\savehyperref{#1}{Definition~\ref*{#1}}}
\newrefformat{line}{\savehyperref{#1}{Line~\ref*{#1}}}
\newrefformat{thm}{\savehyperref{#1}{Theorem~\ref*{#1}}}
\newrefformat{corr}{\savehyperref{#1}{Corollary~\ref*{#1}}}
\newrefformat{cor}{\savehyperref{#1}{Corollary~\ref*{#1}}}
\newrefformat{sec}{\savehyperref{#1}{Section~\ref*{#1}}}
\newrefformat{subsec}{\savehyperref{#1}{Section~\ref*{#1}}}
\newrefformat{app}{\savehyperref{#1}{Appendix~\ref*{#1}}}
\newrefformat{ass}{\savehyperref{#1}{Assumption~\ref*{#1}}}
\newrefformat{asmp}{\savehyperref{#1}{Assumption~\ref*{#1}}}
\newrefformat{ex}{\savehyperref{#1}{Example~\ref*{#1}}}
\newrefformat{fig}{\savehyperref{#1}{Figure~\ref*{#1}}}
\newrefformat{alg}{\savehyperref{#1}{Algorithm~\ref*{#1}}}
\newrefformat{rem}{\savehyperref{#1}{Remark~\ref*{#1}}}
\newrefformat{conj}{\savehyperref{#1}{Conjecture~\ref*{#1}}}
\newrefformat{prop}{\savehyperref{#1}{Proposition~\ref*{#1}}}
\newrefformat{proto}{\savehyperref{#1}{Protocol~\ref*{#1}}}
\newrefformat{prob}{\savehyperref{#1}{Problem~\ref*{#1}}}
\newrefformat{claim}{\savehyperref{#1}{Claim~\ref*{#1}}}
\newrefformat{que}{\savehyperref{#1}{Question~\ref*{#1}}}
\newrefformat{op}{\savehyperref{#1}{Open Problem~\ref*{#1}}}
\newrefformat{fn}{\savehyperref{#1}{Footnote~\ref*{#1}}}
\newrefformat{tab}{\savehyperref{#1}{Table~\ref*{#1}}}
\newrefformat{fig}{\savehyperref{#1}{Figure~\ref*{#1}}}
\newrefformat{proc}{\savehyperref{#1}{Procedure~\ref*{#1}}}
\newrefformat{property}{\savehyperref{#1}{Property~\ref*{#1}}}

%% file: macro_fan.tex
\newcommand{\paren}[1]{{\left( #1 \right)}}
\newcommand{\brac}[1]{{\left[ #1 \right]}}

\newcommand{\defeq}{:=}
\newcommandx{\whp}[1][1=\delta]{with probability at least $1-#1$}
\newcommandx{\Whp}[1][1=\delta]{With probability at least $1-#1$}

\DeclarePairedDelimiter{\set}{\{}{\}}

\newcommand{\R}{\mathbb{R}}

\newcommand{\RNum}[1]{\uppercase\expandafter{\romannumeral #1\relax}}

\newcommand{\cA}{\mathcal{A}}

\newcommand{\cD}{\mathcal{D}}

\newcommand{\cF}{\mathcal{F}}

\newcommand{\cS}{\mathcal{S}}

\newcommand{\abs}[1]{\left|#1\right|}

\newcommand{\nrm}[1]{\left\|#1\right\|}
\newcommand{\linf}[1]{\nrm{#1}_{\infty}}

\newcommand{\Ncov}{\mathcal{N}}

\newcommand{\tbO}{\widetilde{O}}

%% file: sections/related-work.tex
\section{Related Works}
\label{sec:related}

\paragraph{Value-based RL}
Value-based methods are arguably the oldest and most widely studied concepts of reinforcement learning (RL) algorithms \citep{bellman1957markovian,samuel1959some}. They seek to learn an approximation of the optimal state-action-value function $Q^\star$ and act greedily with respect to it, in contrast to policy-gradient methods that directly optimize a parameterized policy. Early works such as Q-learning \citep{watkins1992q}, SARSA \citep{sutton1998reinforcement}, Approximate Dynamic Programming \citep{bertsekas1996neuro} established the foundations, while the successive studies introduced function approximation \citep{sutton1988learning,bertsekas1996neuro}, eligibility traces \citep{singh1996reinforcement}, and residual updates \citep{baird1995residual}.
The combination of value-based ideas with deep neural networks culminated in the Deep Q-Network \citep{mnih2015human}, which sparked a wave of extensions including Double DQN \citep{hasselt2010double,van2016deep}, distributional learning \citep{bellemare2017distributional}, and the integrative DQN-based agent \citep{hessel2018rainbow}.

\paragraph{KL-regularized RL}
KL-regularized (or entropy-regularized) reinforcement learning (RL) originated from the maximum-entropy formulation of \citet{ziebart2008maximum,ziebart2010modeling,neu2017unified}, where a Kullback--Leibler (KL) penalty encourages policies to stay close to a reference distribution while optimizing reward. Different styles of algorithms have emerged from this line of work, including Soft Q-Learning (SQL) style algorithms such as SQL itself \citep{haarnoja2017reinforcement,schulman2017equivalence,guo2022efficient}, Soft Actor-Critic (SAC) style algorithms like PCL \citep{nachum2017bridging}, SAC \citep{haarnoja2018soft}, DQO \citep{ji2024enhancing}, and OREO \citep{wang2024offline}, Point-Wise Direct Alignment Algorithms (DAA-pt) such as KTO \citep{ethayarajh2024kto}, DRO \citep{richemond2024offline}, and an online policy mirror descent variant \citep{team2025kimi}, and Pair-Wise Direct Alignment Algorithms (DAA-pair) like DPO \citep{rafailov2023direct,rafailov2024r} and IPO \citep{azar2024general}. In \cref{sec:comparison}, we provide a comprehensive discussion of the differences between TBRM and these related algorithms.

\paragraph{RL training for LLM reasoning}
RL has played a pivotal role in the post-training of LLMs. The most prominent early example is reinforcement learning from human feedback \citep[RLHF;][]{ouyang2022training,bai2022training}, which uses PPO to align LLMs with human preferences. A series of subsequent works introduced contrastive learning objectives based on pairwise datasets \citep{rafailov2024r,zhao2023slic,azar2024general,tang2024generalized}, or verification-driven objectives using a binary verifier \citep{ethayarajh2024kto,chen2025bridging}. The release of OpenAI's O1 \citep{jaech2024openai} and DeepSeek's R1 \citep{guo2025deepseek} marked a new era of RL algorithms for LLMs---particularly for reasoning tasks---by framing the response generation process as a Markov Decision Process (MDP) and using rule-based verifiers to provide reward signals. Numerous studies have demonstrated and analyzed the effectiveness of RL algorithms in enhancing LLM reasoning capabilities, with PPO and GRPO emerging as the most widely adopted approaches. Prior studies, such as SimpleRL-Zoo \citep{zeng2025simplerl}, Open Reasoner Zero \citep{hu2025openreasonerzeroopensourceapproach}, Light-r1 \citep{wen2025light}, Logic-rl \citep{xie2025logic}, and Skywork-OR1 \citep{skywork-or1-2025}, fall in this category. Variants of these algorithms have been proposed to further improve performance \citep{zhang2025100}. For instance, DAPO \citep{yu2025dapo} enhances GRPO with techniques like clip-higher, dynamic sampling, and token-level policy gradient loss, achieving strong results on AIME24. Dr.~GRPO \citep{liu2025understanding} addresses optimization bias in GRPO to improve token efficiency, while CPPO \citep{lin2025cppo} reduces its computational cost by skipping rollouts with low advantages. VC-PPO \citep{yuan2025s} resolves PPO's challenges with value initialization bias and delayed reward signals through value pretraining and decoupled-GAE. Building on this, VAPO \citep{yuan2025vapo} improve DAPO further by incorporating selected techniques from VC-PPO. Additionally, several works explore REINFORCE \citep{williams1991function,williams1992simple} style RL algorithms, including ReMax \citep{li2023remax}, REINFORCE++ \citep{hu2025reinforce++}, RAFT++ \citep{xiong2025minimalist}, and RLOO \citep{kool2019buy,ahmadian2024back}. However, all these approaches rely on policy-based methods or their variants. In contrast, our method adopts a value-based, off-policy RL approach that is principled, efficient, and theoretically grounded.

%% file: sections/comparison.tex
\section{Comparison with Related Algorithms}
\label{sec:comparison}

In this section, we compare TBRM with other related algorithms in detail.
We group the related algorithms into the following categories:
\begin{itemize}
    \item \textbf{Soft Q-Learning (SQL) Style}: SQL \citep{haarnoja2017reinforcement,schulman2017equivalence,guo2022efficient}
    \item \textbf{Soft Actor-Critic (SAC) Style}: PCL \citep{nachum2017bridging}, SAC \citep{haarnoja2018soft}, DQO \citep{ji2024enhancing}, OREO \citep{wang2024offline}
    \item \textbf{Point-Wise Direct Alignment Algorithms (DAA-pt)}: KTO \citep{ethayarajh2024kto}, DRO \citep{richemond2024offline}, online policy mirror decent variant \citep{team2025kimi}
    \item \textbf{Pair-Wise Direct Alignment Algorithms (DAA-pair)}: DPO \citep{rafailov2023direct,rafailov2024r}, IPO \citep{azar2024general}
\end{itemize}
It is important to note that the present paper primarily addresses LLM reasoning in environments where the state space is tokenized and the base model operates autoregressively. Several algorithms mentioned above were initially developed for continuous control domains such as robotics; however, our analysis considers only their adaptation to the discrete, tokenized setting relevant to language models, as in this paper. Given space constraints, we restrict our discussion to the fundamental principles underlying each algorithmic category rather than providing exhaustive implementation details.

We first present \cref{tab:comparison} to summarize the key differences between TBRM and other algorithms. Note that this comparison is only for algorithm design; the consequences of these differences for theoretical guarantees are likely to be more significant. However, given that TBRM is the only algorithm here with established finite-sample guarantees under the more general MDP setting, we will leave the theoretical comparison to future work.

\begin{table*}[ht]
    \centering
    \begin{adjustbox}{width=\linewidth}
    \begin{tabular}{lcccc}
    \hline
    \bf Algorithms & \bf Optimization & \bf Single~Rollout & \bf Single~Model~Training & \bf Traj.~Reward~Allowed
    \\
    \hline
    \bf SQL        & {\color{Crimson}\xmark}            & {\color{LimeGreen}\cmark}              & {\color{LimeGreen}\cmark} & {\color{Crimson}\xmark}
    \\
    \bf SAC        & {\color{LimeGreen}\cmark}            & {\color{LimeGreen}\cmark}              & {\color{Crimson}\xmark} & {\color{Crimson}\xmark}
    \\
    \bf DAA-pt     & {\color{LimeGreen}\cmark}            & {\color{Crimson}\xmark}              & {\color{LimeGreen}\cmark} & {\color{LimeGreen}\cmark}
    \\
    \bf DAA-pair   & {\color{LimeGreen}\cmark}            & {\color{Crimson}\xmark}              & {\color{LimeGreen}\cmark} & {\color{LimeGreen}\cmark}
    \\
    \rowcolor{bgcolor} \bf TBRM       & {\color{LimeGreen}\cmark}   & {\color{LimeGreen}\cmark}     & {\color{LimeGreen}\cmark} & {\color{LimeGreen}\cmark}
    \\ \hline
    \end{tabular}
    \end{adjustbox}
\caption{Comparison between TBRM and related algorithms in terms of algorithm design.}
\label{tab:comparison}
\end{table*}

For the ease of comparison, we rewrite the loss function of TBRM as follows, by the definition of the autoregressive function approximation:
\begin{align*}
\Lcal_{\Dcal}^{\sf TBRM}({\color{Crimson}\theta}) = &~ \frac{1}{|\Dcal|}\sum_{\tau\in\Dcal} \left(\logit_{\color{Crimson}\theta}(s_1,a_1) - \frac{r(\tau)}{\beta} - \log\piref(\tau) + \sum_{h=2}^{H}\log\pi_{\color{Crimson}\theta}(a_{h} \mid s_{h}) \right)^2
\\
= &~ \frac{1}{|\Dcal|}\sum_{\tau\in\Dcal} \left(V_{\color{Crimson}\theta}(s_1) - \frac{r(\tau)}{\beta} + \sum_{h=1}^{H}\log\frac{\pi_{\color{Crimson}\theta}(a_{h} \mid s_{h})}{\piref(a_{h} \mid s_{h})} \right)^2.
\end{align*}

\paragraph{Comparing TBRM with Soft Q-Learning Style Algorithms}
The soft Q-learning based algorithms are typically iterative algorithms, formulated as two different versions: single-step case and multi-step case. We consider the loss used in  \citet{guo2022efficient} which is motivated by \emph{path consistency learning} \citep[PCL;][]{nachum2017bridging}. The single-step case is then formulated as
\begin{align*}
\theta_{t+1} \leftarrow \argmin_\theta \Lcal_{\Dcal}^{\sf SQL\text{-}s}({\color{Crimson}\theta}; {\color{DodgerBlue}\theta_t}) \coloneqq \frac{1}{|\Dcal|} \sum_{\tau\in\Dcal} \sum_{h=1}^{H} \left(V_{\color{DodgerBlue}\theta_t}(s_h) - \frac{r(s_h,a_h)}{\beta} + \log\frac{\pi_{\color{Crimson}\theta}(a_{h} \mid s_{h})}{\piref(a_{h} \mid s_{h})} - V_{\color{DodgerBlue}\theta_t}(s_{h+1}) \right)^2,
\end{align*}
while the multi-step case is formulated as
\begin{align*}
    \theta_{t+1} \leftarrow \argmin_\theta \Lcal_{\Dcal}^{\sf SQL\text{-}m}({\color{Crimson}\theta}; {\color{DodgerBlue}\theta_t}) \coloneqq \frac{1}{|\Dcal|} \sum_{\tau\in\Dcal} \sum_{h=1}^{H} \left(V_{\color{DodgerBlue}\theta_t}(s_h) - \sum_{h'=h}^{H}\frac{r(s_{h'},a_{h'})}{\beta} + \sum_{h'=h}^{H}\log\frac{\pi_{\color{Crimson}\theta}(a_{h'} \mid s_{h'})}{\piref(a_{h'} \mid s_{h'})} \right)^2.
\end{align*}

If we want to exactly match the original loss of soft Q-learning \citep{haarnoja2017reinforcement,schulman2017equivalence}, then these should be rewritten as
\neurips{\begin{small}}
\begin{gather*}
    \theta_{t+1} \leftarrow \argmin_\theta \widetilde\Lcal_{\Dcal}^{\sf SQL\text{-}s}({\color{Crimson}\theta}; {\color{DodgerBlue}\theta_t}) \coloneqq \frac{1}{|\Dcal|} \sum_{\tau\in\Dcal} \sum_{h=1}^{H} \left(\logit_{\color{Crimson}\theta}(s_h) - \frac{r(s_h,a_h)}{\beta} - \log\piref(a_{h} \mid s_{h}) - V_{\color{DodgerBlue}\theta_t}(s_{h+1}) \right)^2,
    \\
    \theta_{t+1} \leftarrow \argmin_\theta \widetilde\Lcal_{\Dcal}^{\sf SQL\text{-}m}({\color{Crimson}\theta}; {\color{DodgerBlue}\theta_t}) \coloneqq \frac{1}{|\Dcal|} \sum_{\tau\in\Dcal} \sum_{h=1}^{H} \left(\logit_{\color{Crimson}\theta}(s_h) - \sum_{h'=h}^{H}\frac{r(s_{h'},a_{h'})}{\beta} - \log\piref(a_{h} \mid s_{h}) + \sum_{h'=h+1}^{H}\log\frac{\pi_{\color{DodgerBlue}\theta_t}(a_{h'} \mid s_{h'})}{\piref(a_{h'} \mid s_{h'})} \right)^2.
\end{gather*}
\neurips{\end{small}}
There can be a more general multi-step version, which blends the multi-step return in $\Lcal_{\Dcal}^{\sf SQL\text{-}m}$ and the value bootstrap in $\Lcal_{\Dcal}^{\sf SQL\text{-}s}$, but we omit it here for brevity as our existing argument would directly extend to this case.

From the derivation above, we can identify two key distinctions between TBRM and soft Q-learning based algorithms: 1) TBRM employs direct optimization rather than an iterative approach, and 2) TBRM's loss function operates on complete trajectories rather than summing losses over individual timesteps within trajectories, hence eliminating the need for per-step reward.

\paragraph{Comparing TBRM with Soft Actor-Critic Style Algorithms}
The soft actor-critic style algorithms for LLMs are similar to SQL, but they 1) introduce a separate $V$ model; 2) operate as optimization rather than iteration. In particular, in the single-step case, 
\begin{align*}
\argmin_{\theta,\phi} \Lcal_{\Dcal}^{\sf SAC\text{-}s}({\color{Crimson}\theta}, {\color{DodgerBlue}\phi}) \coloneqq \frac{1}{|\Dcal|} \sum_{\tau\in\Dcal} \sum_{h=1}^{H} \left(V_{\color{DodgerBlue}\phi}(s_h) - \frac{r(s_h,a_h)}{\beta} + \log\frac{\pi_{\color{Crimson}\theta}(a_{h} \mid s_{h})}{\piref(a_{h} \mid s_{h})} - V_{\color{DodgerBlue}\phi}(s_{h+1}) \right)^2,
\end{align*}
while the multi-step case is
\begin{align*}
    \argmin_{\theta,\phi} \Lcal_{\Dcal}^{\sf SAC\text{-}m}({\color{Crimson}\theta}, {\color{DodgerBlue}\phi}) \coloneqq \frac{1}{|\Dcal|} \sum_{\tau\in\Dcal} \sum_{h=1}^{H} \left(V_{\color{DodgerBlue}\phi}(s_h) - \sum_{h'=h}^{H}\frac{r(s_{h'},a_{h'})}{\beta} + \sum_{h'=h}^{H}\log\frac{\pi_{\color{Crimson}\theta}(a_{h'} \mid s_{h'})}{\piref(a_{h'} \mid s_{h'})} \right)^2.
\end{align*}

\paragraph{Comparing TBRM with Point-Wise Direct Alignment Algorithms}

Perhaps surprisingly, among all four categories of algorithms, the point-wise direct alignment algorithms appear to be the most similar to TBRM, although they are derived from a  different perspective (mostly from bandit formulation). We view the core objective of these algorithms as optimizing the following loss
\begin{align*}
    \argmin_{\theta} \Lcal_{\Dcal}^{\sf DAA\text{-}pt}({\color{Crimson}\theta})
    = \frac{1}{|\Dcal|}\sum_{\tau\in\Dcal} \left(\widehat V_\beta^\star(s_1) - \frac{r(\tau)}{\beta} + \sum_{h=1}^{H}\log\frac{\pi_{\color{Crimson}\theta}(a_{h} \mid s_{h})}{\piref(a_{h} \mid s_{h})} \right)^2,
\end{align*}
where $\widehat V_\beta^\star$ is an estimate of $V_\beta^\star$ (the optimal value function for KL-regularized MDP).
Note that $V_\beta^\star$ is exactly the same as the partition function in the bandit formulation. One popular way to estimate $\widehat V_\beta^\star$ is by using the softmax of returns from multiple rollouts for each question,\footnote{$\widehat V_\beta^\star$ can also be estimated using a separate model \citep{richemond2024offline}, similar to our discussion of SAC-style algorithms.} for example \citep{team2025kimi},
\begin{align*}
\widehat V_\beta^\star(s_1) \leftarrow \beta \log\sum_{\tau \sim \pi_{\theta} \mid s_1}\exp\left( \frac{r(\tau)}{\beta} \right).
\end{align*}
However, it is unclear whether this estimate is accurate enough, particularly when the rollout policy $\pi_{\theta}$ differs significantly from the optimal policy $\pi^\star_\beta$.
In contrast, TBRM leverages 1) Bellman equation in KL-regularized RL (see, e.g., \cref{sec:arfa}) and 2) recent advances in change of trajectory measure, which allows us to directly use $V_{\theta}$ instead of requiring $\widehat V_\beta^\star$. This approach enables TBRM to provably converge to a near-optimal policy using only a single rollout per prompt, while maintaining the advantages of a direct optimization algorithm.

\paragraph{Comparing TBRM with Pair-Wise Direct Alignment Algorithms}
Under the perspective above, we can view the motivation of pair-wise direct alignment algorithms as using a pair of responses from the same question $s_1$ to cancel the need for $V_{\theta}(s_1)$, which leads to the following objective:
\begin{align*}
    \argmin_{\theta} \Lcal_{\Dcal}^{\sf DAA\text{-}pair}({\color{Crimson}\theta})
    = \frac{1}{|\Dcal|}\sum_{(\tau, \tau')\in\Dcal} \left(\sum_{h=1}^{H}\log\frac{\pi_{\color{Crimson}\theta}(a_{h} \mid s_{h})}{\piref(a_{h} \mid s_{h})} - \sum_{h=1}^{H}\log\frac{\pi_{\color{Crimson}\theta}(a'_{h} \mid s'_{h})}{\piref(a'_{h} \mid s'_{h})} - \frac{r(\tau)}{\beta} + \frac{r(\tau')}{\beta}\right)^2.
\end{align*}
Comparing with TBRM, the pair-wise direct alignment algorithms are basically optimizing the difference between the Bellman residuals of two trajectories.

%% file: sections/proof-appendix.tex
\section{Technical Tools}\label[appendix]{app: perf-diff}

We now present \cref{lem:soft_pd} as a soft performance difference lemma with arbitrary reference function. 
\cref{lem:soft_pd} holds generally for KL-regularized RL, which uses a slightly different form of Bellman operator defined as below,
\begin{align*}
(\Tcal^\pi_\beta f)(s_h,a_h) \coloneqq &~ r(s_h,a_h) + \E_{s_{h+1} \sim \Pcal(\cdot \mid s_h,a_h), a_{h+1} \sim \pi(\cdot \mid s_{h+1})} \left[ f(s_{h+1},a_{h+1}) - \beta \log\frac{\pi(a_{h+1} \mid s_{h+1})}{\piref(a_{h+1} \mid s_{h+1})} \right].
\end{align*}
We present this slightly different version as this lemma is more generally applicable than the results in this paper.
\begin{lemma}[Soft Performance Difference Lemma via Reference Function]
\label[lemma]{lem:soft_pd}
For any function $f$ as well as any policies $\pi^\dag$ and $\pi$, we have
\begin{align*}
J_\beta(\pi^\dag) - J_\beta(\pi) = &~ \E_{\pi^\dag}\left[\sum_{h=1}^{H} \left( (\Tcal_\beta^{\pi} f) (s_h,a_h) - f(s_h,a_h) \right)\right] + \E_{\pi}\left[\sum_{h=1}^{H} \left( f(s_h,a_h) - (\Tcal_\beta^{\pi} f) (s_h,a_h) \right)\right]
\\
&~ + \E_{\pi^\dag} \left[\sum_{h=1}^H \left( f(s_h,a_h) - \beta \log\frac{\pi^\dag(a_h \mid s_h)}{\piref(a_h \mid s_h)} - \E_{a_h \sim \pi(\cdot \mid s_h)} \left[ f(s_h,a_h) - \beta \log\frac{\pi(a_h \mid s_h)}{\piref(a_h \mid s_h)} \right] \right) \right].
\end{align*}
\end{lemma}
\begin{proof}[\cpfname{lem:soft_pd}]
We first prove the soft performance difference via $Q_\beta^\pi$ as the reference,
\begin{align*}
&~ J_\beta(\pi^\dag) - J_\beta(\pi)
\\
= &~ \E_{\pi^\dag} \left[\sum_{h=1}^H \left( r(s_h,a_h) - \beta \log\frac{\pi^\dag(a_h \mid s_h)}{\piref(a_h \mid s_h)} \right)\right] - J_\beta(\pi)
\\
= &~ \E_{\pi^\dag} \left[\sum_{h=1}^H \left( r(s_h,a_h) - \beta \log\frac{\pi^\dag(a_h \mid s_h)}{\piref(a_h \mid s_h)} - V_\beta^\pi(s_h) + V_\beta^\pi(s_h) \right)\right] - J_\beta(\pi)
\\
= &~ \E_{\pi^\dag} \left[\sum_{h=1}^H \left( r(s_h,a_h) - \beta \log\frac{\pi^\dag(a_h \mid s_h)}{\piref(a_h \mid s_h)} - V_\beta^\pi(s_h) + V_\beta^\pi(s_h) \right)\right] - \E_{s_1 \sim \rho} \left[ V_\beta^\pi(s_1) \right]
\\
= &~ \E_{\pi^\dag} \left[\sum_{h=1}^H \left( r(s_h,a_h) - \beta \log\frac{\pi^\dag(a_h \mid s_h)}{\piref(a_h \mid s_h)} - V_\beta^\pi(s_h) + V_\beta^\pi(s_h) \right) - V_\beta^\pi(s_1)\right]
\\
= &~ \E_{\pi^\dag} \left[\sum_{h=1}^H \left( r(s_h,a_h) + \E_{s_{h+1} \sim \Pcal(\cdot \mid s_h,a_h)} \left[ V_\beta^\pi(s_{h+1}) \right] - \beta \log\frac{\pi^\dag(a_h \mid s_h)}{\piref(a_h \mid s_h)} - V_\beta^\pi(s_h) \right)\right]
\\
= &~ \E_{\pi^\dag} \left[\sum_{h=1}^H \left( Q_\beta^\pi(s_h,a_h) - \beta \log\frac{\pi^\dag(a_h \mid s_h)}{\piref(a_h \mid s_h)} - \E_{a_h \sim \pi(\cdot \mid s_h)} \left[ Q_\beta^\pi(s_h,\pi) + \beta \log\frac{\pi(a_h \mid s_h)}{\piref(a_h \mid s_h)} \right] 
    \right) \right].
\end{align*}
Next, consider an arbitrary reference function $f$, we define the augment reward $r_f$ as
\begin{align*}
r_f(s_h,a_h) \coloneqq &~ f(s_h,a_h) - \E_{s_{h+1} \sim \Pcal(\cdot \mid s_h,a_h), a_{h+1} \sim \pi(\cdot \mid s_{h+1})} \left[ f(s_{h+1},a_{h+1}) - \beta \log\frac{\pi(a_{h+1} \mid s_{h+1})}{\piref(a_{h+1} \mid s_{h+1})} \right]
\\
= &~ f(s_h,a_h) - (\Tcal^\pi_\beta f)(s_h,a_h) + r(s_h,a_h),
\end{align*}
for any $(s_h,a_h) \in \Scal_h \times \Acal_h$. This means $f$ is the fixed point of $\Tcal^\pi_\beta$ with replacing reward function $r$ by $r_f$. We use $Q_{\beta,r_f}^\pi$ to denote the soft Q-function with replacing reward function $r$ by $r_f$, and we immediately have $f \equiv Q_{\beta,r_f}^\pi$. We also use $J_{\beta,r_f}(\pi)$ to denote the soft return of policy $\pi$ with replacing reward function $r$ by $r_f$.
Then
\begin{align*}
&~ J_\beta(\pi^\dag) - J_\beta(\pi)
\\
= &~ J_\beta(\pi^\dag) - J_{\beta,r_f}(\pi^\dag) + J_{\beta,r_f}(\pi^\dag) - J_{\beta,r_f}(\pi) + J_{\beta,r_f}(\pi) - J_\beta(\pi)
\\
= &~ \E_{\pi^\dag} \left[\sum_{h=1}^H \left( r(s_h,a_h) - r_f(s_h,a_h) \right) \right] + \E_{\pi} \left[\sum_{h=1}^H \left( r_f(s_h,a_h) - r(s_h,a_h) \right)\right]  \\
&~ + \E_{\pi^\dag} \left[\sum_{h=1}^H \left( f(s_h,a_h) - \beta \log\frac{\pi^\dag(a_h \mid s_h)}{\piref(a_h \mid s_h)} - \E_{a_h \sim \pi(\cdot \mid s_h)} \left[ f(s_h,a_h) - \beta \log\frac{\pi(a_h \mid s_h)}{\piref(a_h \mid s_h)} \right] \right) \right]
\tag{by the soft performance difference via reference as $f \equiv Q_{\beta,r_f}^\pi$}
\\
= &~ \E_{\pi^\dag} \left[\sum_{h=1}^H \left( (\Tcal^\pi_\beta f)(s_h,a_h) - f(s_h,a_h) \right)\right] + \sum_{h=1}^H \E_{d^{\pi}_h} \left[ f(s_h,a_h) - (\Tcal^\pi_\beta f)(s_h,a_h) \right]
\\
&~ + \E_{\pi^\dag} \left[\sum_{h=1}^H \left( f(s_h,a_h) - \beta \log\frac{\pi^\dag(a_h \mid s_h)}{\piref(a_h \mid s_h)} - \E_{a_h \sim \pi(\cdot \mid s_h)} \left[ f(s_h,a_h) - \beta \log\frac{\pi(a_h \mid s_h)}{\piref(a_h \mid s_h)} \right] \right)\right].
\end{align*}
This completes the proof.
\end{proof}
When specializing to the Bellman operator in this paper as in \cref{sec:arfa}, \cref{lem:soft_pd} becomes
\begin{align*}
    \frac{J_\beta(\pi^\dag) - J_\beta(\pi)}{\beta}
    = &~ \E_{\pi^\dag}\left[\sum_{h=1}^{H} \left( (\Tcal_\beta^{\pi} Q) (s_h,a_h) - Q(s_h,a_h) \right)\right] + \E_{\pi}\left[\sum_{h=1}^{H} \left( Q(s_h,a_h) - (\Tcal_\beta^{\pi} Q) (s_h,a_h) \right)\right]
    \\
    &~ + \E_{\pi^\dag} \left[\sum_{h=1}^H \left( Q(s_h,a_h) - \log\pi^\dag(a_h \mid s_h) - \E_{a_h \sim \pi(\cdot \mid s_h)} \left[ Q(s_h,a_h) - \log\pi(a_h \mid s_h) \right] \right) \right].
\end{align*}
As corollary, we have the following upper bound on the sub-optimality of any policy induced by a value function.
\begin{corollary}\label[corollary]{cor:all-pol-sub-opt}
Suppose that the policy $\pihat=\pi_Q$ is induced by a value function $Q$, i.e.,
\begin{align*}
    \pihat(a\mid s)=\frac{\exp\paren{Q(s,a)}}{\sum_{a'\in\cA} \exp\paren{Q(s,a')}}.
\end{align*}
Then it holds that
\begin{align*}
    J_\beta(\pi^\star) - J_\beta(\pihat)\leq 2 \beta \max_{\pi\in\set{\pi^\star,\pihat}} \left| \E_{\pi}\left[\sum_{h=1}^{H} \left( Q(s_h,a_h) - \left(\Tcal_\beta Q\right) (s_h,a_h) \right)\right] \right|
\end{align*}
\end{corollary}

\begin{proof}[\cpfname{cor:all-pol-sub-opt}]
By \cref{lem:soft_pd}, we have
\begin{align*}
&~ J_\beta(\pi^\star) - J_\beta(\pihat)
\\
= &~ \beta \E_{\pi^\star}\left[\sum_{h=1}^{H} \left( (\Tcal_\beta^{\pihat} Q) (s_h,a_h) - Q(s_h,a_h) \right)\right] + \beta \E_{\pihat}\left[\sum_{h=1}^{H} \left( Q(s_h,a_h) - (\Tcal_\beta^{\pihat} Q) (s_h,a_h) \right)\right]
\\
&~ + \beta \E_{\pi^\star} \left[\sum_{h=1}^H \left( Q(s_h,a_h) - \log\pi^\star(a_h \mid s_h) - \E_{a_h' \sim \pihat(\cdot \mid s_h)} \left[ Q(s_h,a_h') - \log\pihat(a_h' \mid s_h) \right] \right) \right]
\\
\leq &~ \beta \E_{\pi^\star}\left[\sum_{h=1}^{H} \left( \left(\Tcal_\beta Q\right) (s_h,a_h) - Q(s_h,a_h) \right)\right] + \beta \E_{\pihat}\left[\sum_{h=1}^{H} \left( Q(s_h,a_h) - \left(\Tcal_\beta Q\right) (s_h,a_h) \right)\right]
\\
\leq &~ 2 \beta \max_{\pi} \left| \E_{\pi}\left[\sum_{h=1}^{H} \left( Q(s_h,a_h) - \left(\Tcal_\beta Q\right) (s_h,a_h) \right)\right] \right| ,
\end{align*}
where the first inequality follows from the definition that 
\begin{align*}
     \pihat(\cdot \mid s_h)=\argmax_{p\in\Delta(\cA)}\E_{a_h' \sim p} \left[ Q(s_h,a_h') - \log p(a_h') \right], \quad \forall s_h\in\cS_h,
\end{align*}
and hence
\begin{align*}
    \E_{a_h \sim \pi^\star(\cdot \mid s_h)} \left[Q(s_h,a_h) - \log\pi^\star(a_h \mid s_h) \right]
    \leq \E_{a_h' \sim \pihat(\cdot \mid s_h)} \left[ Q(s_h,a_h') - \log\pihat(a_h' \mid s_h) \right].
\end{align*}
\end{proof}

\section{Change of Trajectory Measure with Concentrability}
\label[appendix]{appendix:cotm_proof}
We first state the following lemma, which follows immediately from the definition of $\chi^2$-divergence, see e.g. \cite[Eq. (7.91)]{polyanskiy2025information}
\begin{lemma}\label[lemma]{lem:chi-square}
For distribution $P, Q\in\Delta(\mathcal{X})$ and function $F:\mathcal{X}\to \RR$ such that $\EE_{Q}[F(X)]=0$, it holds that
\begin{align*}
    \left(\EE_P F(X) \right)^2\leq \chisq{P}{Q}\cdot \EE_Q F(X)^2.
\end{align*}
\end{lemma}

\begin{proof}[\cpfname{lem: change-of-trajectory-measure}]
    We only need to prove that for any function $f: \calS\times \calA\to \RR$, 
    \begin{equation}\label{eq: change-traj-eq1}
        \left(\EE_{\pi}\left[\sum_{h=1}^H f(s_h, a_h)\right]\right)^2\le \left(1+\sum_{h=1}^H \chisq{d_h^\pi}{d_h^{\piref}}\right)\cdot \EE_{\piref}\left[\left(\sum_{h=1}^H f(s_h, a_h)\right)^2\right].
    \end{equation}

    First of all, we construct function $\barf: \calS\to \RR$ as 
    \begin{equation}\label{eq: def-barf}
        \barf(s_h)\coloneqq \EE_{\mu}\cond{\sum_{h'=h}^H f(s_{h'}, a_{h'})}{s_h},\qquad \forall h\in[H], s_h\in\calS_h.
    \end{equation}
    Then, it is direct to verify that for $h\in[H]$, $s_h\in\calS_h$,
    \begin{align*}
        \barf(s_h)\coloneqq \EE_{\mu}\cond{f(s_h,a_h)+\barf(s_{h+1})}{s_h},
    \end{align*}
    where we adopt the notation that $s_{H+1}$ is a deterministic terminal state and $\barf(s_{H+1})=0.$ Then, we expand
    \begin{align*}
        &~\EE_{\mu} \left[\left(\sum_{h=1}^H f(s_h, a_h)\right)^2\right] \\
        =&~\EE_{\mu} \left[\left(\barf(s_1)+\sum_{h=1}^H \brac{ f(s_h, a_h)+\barf(s_{h+1})-\barf(s_h) }\right)^2\right] \\
        =&~ \EE_{\mu} \left[\barf(s_1)^2+\sum_{h=1}^H \left(f(s_h, a_h)+\barf(s_{h+1})-\barf(s_h) \right)^2 \right] \\
        &~+ \EE_{\mu} \left[\sum_{1\leq h\leq H } \barf(s_1)\left(f(s_h, a_h)+\barf(s_{h+1})-\barf(s_h) \right)\right]\\
        &~ +\EE_{\mu} \left[\sum_{1\leq h'<h\leq H } \left(f(s_{h'}, a_{h'})+\barf(s_{h'+1})-\barf(s_{h'}) \right)\left(f(s_h, a_h)+\barf(s_{h+1})-\barf(s_h) \right)\right].
    \end{align*}
    Therefore, using the Markov property, it holds that for any $h\in[H]$,
    \begin{align*}
        \EE_{\mu}\left[f(s_h, a_h)+\barf(s_{h+1})-\barf(s_h)|s_1,a_1,\cdots,s_h\right]=\EE_{\mu}\cond{f(s_h,a_h)+\barf(s_{h+1})}{s_h}-\barf(s_h)=0,
    \end{align*}
    and hence we can deduce that
    \begin{align}\label{eq: fbar-f}
        \EE_{\mu} \left[\left(\sum_{h=1}^H f(s_h, a_h)\right)^2\right]=\EE_{\mu} \left[\barf(s_1)^2+\sum_{h=1}^H \left(f(s_h, a_h)+\barf(s_{h+1})-\barf(s_h) \right)^2 \right].
    \end{align}
    Next, for every $h\in[H]$, we apply \cref{lem:chi-square} on the function $(s_h,a_h,s_{h+1})\mapsto f(s_h, a_h) - \barf(s_h) + \barf(s_{h+1})$ to derive
    \begin{align}\label{eq: fbar-f-chisq}
    \begin{aligned}
        \MoveEqLeft
        \left(\EE_{\pi}\left[f(s_h, a_h) - \barf(s_h) + \barf(s_{h+1})\right]\right)^2 \\
        \leq&~ \chisq{d^\pi_h}{d^{\mu}_h}\cdot \EE_{\mu}\left[\left(f(s_h, a_h) - \barf(s_h) + \barf(s_{h+1})\right)^2\right],
    \end{aligned}
    \end{align}
    where we again use the fact that $\EE_{\mu}\left[f(s_h, a_h)+\barf(s_{h+1})-\barf(s_h)\right]=0$.
    
    Furthermore, we note that $\EE_{\pi} \barf(s_1)=\EE_{s_1\sim \rho} \barf(s_1)=\EE_{\mu} \barf(s_1)$, and hence
    \begin{align}\label{eq: fbar-s1}
        \left(\EE_{\pi} \barf(s_1)\right)^2=\left(\EE_{\mu} \barf(s_1)\right)^2\leq \EE_{\mu} \barf(s_1)^2.
    \end{align}

    Therefore, combining the inequalities above, we have
    \begin{align*}
        &~\left(\EE_{\pi}\left[\sum_{h=1}^H f(s_h, a_h)\right]\right)^2 \\
        =&~ \left(\barf(s_1)+\sum_{h=1}^H\EE_{\pi}\left[f(s_h, a_h) - \barf(s_h) + \barf(s_{h+1})\right]\right)^2 \\
        \stackrel{(i)}{\leq }&~ \left(\sqrt{\EE_{\mu} \barf(s_1)^2}+\sum_{h=1}^H\sqrt{\chisq{d^\pi_h}{d^{\mu}_h}\cdot \EE_{\mu}\left[\left(f(s_h, a_h) - \barf(s_h) + \barf(s_{h+1})\right)^2\right]} \right)^2 \\
        \stackrel{(ii)}{\leq }&~ \left(1+\sum_{h=1}^H \chisq{d^\pi_h}{d^{\mu}_h} \right) \left(\EE_{\mu} \barf(s_1)^2 + \sum_{h=1}^H \EE_{\mu}\left(f(s_h, a_h) - \barf(s_h) + \barf(s_{h+1})\right)^2 \right) \\
        \stackrel{(iii)}{=}&~ \left(1+\sum_{h=1}^H \chisq{d^\pi_h}{d^{\mu}_h} \right)\cdot \EE_{\mu} \left[\left(\sum_{h=1}^H f(s_h, a_h)\right)^2\right],
    \end{align*}
    where $(i)$ uses \pref{eq: fbar-f-chisq} and \pref{eq: fbar-s1}, $(ii)$ uses Cauchy-Schwarz inequality, and (iii) uses \pref{eq: fbar-f}.
    Hence \pref{eq: change-traj-eq1} is verified.
\end{proof}

Recall the definition of state-action concentrability (e.g. Eq. (5) in \citet{jia2025we}):
$$C_{\mathtt{sa}}(\pi; \mu) = \max_{h\in [H]}\sup_{s_h\in \calS_h, a_h\in \calA} \frac{d_h^\pi(s_h, a_h)}{d_h^{\mu}(s_h, a_h)}.$$
We have following direct corollary of \cref{lem: change-of-trajectory-measure}, after noticing that the $\chi^2$ divergence can be upper bounded by the state concentrability. However, we remark that the state concentrability might yield a much more pessimistic bound compared to the upper bound of \cref{lem: change-of-trajectory-measure}.
\begin{corollary}
\label[corollary]{cor:cotm-cor}
    Given MDP $M = (H, \calS, \calA, \calP, r, \rho)$. We use $d_h^\pi(\cdot), d_h^{\mu}(\cdot)\in \Delta(\calS_h\times\calA)$ to denote the occupancy measure of the MDP under policy $\pi$ and $\mu$. Then we have
    \begin{align*}
        \sup_{f} \frac{\left(\EE_{\pi}\left[\sum_{h=1}^H f(s_h, a_h)\right]\right)^2}{\EE_{\mu}\left[\left(\sum_{h=1}^H f(s_h, a_h)\right)^2\right]}\le H\cdot C_{\mathtt{sa}}(\pi;\mu),
    \end{align*}
    where the supremum is over all functions $f: \calS\times\calA\to \RR$.
\end{corollary}
\begin{proof}[\cpfname{cor:cotm-cor}]
    In view of \pref{lem: change-of-trajectory-measure}, we only need to verify that 
    \begin{equation} \label{eq: chi-sq-cov} 
        1 + \chisq{d_h^\pi}{d_h^{\mu}}\le C_{\mathtt{sa}}.
    \end{equation}
    We have 
    \begin{align*} 
        1 + \chisq{d^\pi_h}{d^{\mu}_h} & = \sum_{s_h\in \calS_h, a_h\in \calA} \frac{d_h^\pi(s_h, a_h)^2}{d^{\mu}_h(s_h, a_h)} \\
        & \le \sum_{s_h\in \calS_h, a_h\in \calA} d^\pi_h(s_h, a_h)\cdot C_{\mathtt{sa}}(\pi; \mu) = C_{\mathtt{sa}}(\pi; \mu),
    \end{align*}
    and \pref{eq: chi-sq-cov} is verified.
\end{proof}

\section{Proof of \pref{thm:tbrm_main}}\label[appendix]{app: proof-tbrm}

\newcommand{\hR}{\widehat{R}}
\newcommandx{\LTBRM}[1][1=\cD]{\Lcal_{#1}^{\sf TBRM}}
Recall that the TBRM loss defined in \cref{eq:brm_mc_obj} is given by
\begin{align}
\LTBRM(\theta) \coloneqq \frac{1}{|\Dcal|}\sum_{\tau\in\Dcal} \left(\logit_\theta(s_1,a_1) - \frac{r(\tau)}{\beta} - \log\piref(\tau) + \sum_{h=2}^{H}\log\pi_\theta(a_{h} \mid s_{h}) \right)^2.
\end{align}
In the following, to simplify presentation, we define
\begin{align*}
    f_{\theta,1}(s_1,a_1)=&~\logit_\theta(s_1,a_1) - \frac{r(s_1,a_1)}{\beta} - \log\piref(a_1 \mid s_1), \\
    f_{\theta,h}(s_h,a_h)=&~\log \pi_\theta(a_h \mid s_h) - \frac{r(s_h,a_h)}{\beta} - \log\piref(a_h \mid s_h), \qquad \forall h>1.
\end{align*}
and 
\begin{align*}
    f_\theta(\tau)\defeq&~ \logit_\theta(s_1,a_1) - \frac{r(\tau)}{\beta} - \log\piref(\tau) + \sum_{h=2}^{H}\log\pi_\theta(a_{h} \mid s_{h}) \\
    =&~ \sum_{h=1}^H f_{\theta,h}(s_h,a_h), \qquad \forall \tau=(s_1,a_1,\cdots,s_H,a_H).
\end{align*}

\paragraph{Uniform convergence} Before applying \cref{lem: change-of-trajectory-measure}, we need to first relate the empirical loss $\LTBRM$ to the population loss. We introduce the following assumption on the parametrization.
\begin{assumption}[Bounded and smooth parametrization]\label[assumption]{asmp:theta-bound}
There exists constant $C_\Theta\geq 1$ and parameter $L_\Theta$ such that for any $\theta\in \Theta$, it holds that $\forall s\in\cS, a\in\cA$,
\begin{align*}
    \abs{\logit_\theta(s,a)}\leq \frac{C_\Theta}{\beta}, \qquad 
    \nrm{\nabla \logit_\theta(s,a)}_*\leq L_\Theta,
\end{align*}
where $\nrm{\cdot}_*$ is the dual norm of $\nrm{\cdot}$. 
We also assume $\abs{\log\piref(a\mid s)}\leq \frac{C_\Theta}{\beta}$ $\forall s\in\cS, a\in\cA$.
\end{assumption}

\begin{lemma}\label[lemma]{lem:TBRM-stat}
Fix $\delta\in(0,1)$.
Suppose that \cref{asmp:theta-bound} holds. Then \whp~(over the randomness of the dataset $\cD$),
\begin{align*}
    \EE_{\tau\sim \mu} f_\theta(\tau)^2 \leq 2\LTBRM(\theta)+ \frac{\epsstat(N)}{\beta^2},
\end{align*}
where the statistical error $\epsstat(N)$ is defined as
\begin{align*}
    \epsstat(N)\defeq c\cdot C_\Theta^2H^2\paren{ \frac{\log (\Ncov(\Theta,\alpha)/\delta)}{N}+HL_\Theta\beta\alpha }
\end{align*}
where $c>0$ is a large absolute constant and $\alpha\geq 0$ is a fixed parameter.
\end{lemma}

The above upper bound can be further simplified by the standard assumption on the covering number of $\Theta$.
\begin{assumption}[Parametric function class]\label{asmp:parametric-theta}
The parameter space $\Theta\subseteq\set{\theta\in\R^d: \nrm{\theta}\leq R}$. In this case, we write $\dim(\Theta)=d$.
\end{assumption}
Under \cref{asmp:parametric-theta}, it is clear that $\log \Ncov(\Theta,\alpha)\leq O\paren{d\log(1/\alpha)}$ for all $\alpha>0$ (see e.g., \citet{wainwright2019high}. Therefore, \cref{lem:TBRM-stat} implies that $\epsstat(N)\asymp \frac{H^2}{N}$ (up to poly-logarithmic factors).

\paragraph{Bounding the sub-optimality}

Under \cref{asmp:realizable}, it is clear that $f_\ths \equiv 0$ and hence $\LTBRM(\ths)=0$ (as we have argued in \cref{sec:tbrm_theory}). Therefore, using the condition that $\Lcal_{\Dcal}^{\sf TBRM}(\htheta)-\inf_{\theta\in\Theta}\Lcal_{\Dcal}^{\sf TBRM}(\theta)\leq \epsopt$, we have
\begin{align*}
    \EE_{\tau\sim \mu} f_{\htheta}(\tau)^2\leq 2\epsopt+\beta^{-2}\epsstat(N).
\end{align*}

In the following, we denote $\pihat\defeq \pi_{\htheta}$, $\Qhat(s,a)\defeq \logit_{\htheta}(s,a)$, and $\Vhat$ be the corresponding value function. Then, by \cref{cor:all-pol-sub-opt}, we have
\begin{align*}
J_\beta(\pi^\star) - J_\beta(\pihat)
\leq &~ 2 \beta \max_{\pi\in\set{\pi^\star,\pihat}} \left| \E_{\pi}\left[\sum_{h=1}^{H} \left( \Qhat(s_h,a_h) - \left(\Tcal_\beta \Qhat\right) (s_h,a_h) \right)\right] \right|.
\end{align*}
Note that the MDP is deterministic, and hence $\left(\Tcal_\beta \Qhat\right) (s_h,a_h)=R_\beta(s_h,a_h)+\Vhat(s_{h+1})$ holds deterministically. Therefore, for any fixed policy $\pi$, we have
\begin{align*}
    &~\E_{\pi}\left[\sum_{h=1}^{H} \left( \Qhat(s_h,a_h) - \left(\Tcal_\beta \Qhat\right) (s_h,a_h) \right)\right] \\
    =&~\E_{\pi}\left[\sum_{h=1}^{H} \left( \Qhat(s_h,a_h)-R_\beta(s_h,a_h)-\Vhat(s_{h+1}) \right)\right] \\
    =&~ \E_{\pi}\left[\Qhat(s_1,a_1)-\sum_{h=1}^{H}R_\beta(s_h,a_h)+\sum_{h=1}^{H} \brac{ \Qhat(s_h,a_h)-\Vhat(s_{h}) }\right] \\
    =&~ \E_{\pi}\left[\Qhat(s_1,a_1)-R_\beta(\tau)+\sum_{h=1}^{H} \log \pi_{\htheta}(a_h\mid s_h)\right]
    = \E_\pi \brac{ f_{\htheta}(\tau) }.
\end{align*}
Further, by \cref{lem: change-of-trajectory-measure}, it holds that for any $\pi=\pi_\theta$,
\begin{align*}
    \paren{ \E_\pi \brac{ f_{\htheta}(\tau) } }^2\leq H\Conchis{\mu} \cdot \E_{\tau\sim \mu} f_{\htheta}(\tau)^2
    \leq H\Conchis{\mu} \paren{2\epsopt+\beta^{-2}\epsstat(N)}.
\end{align*}
Therefore, we can conclude that
\begin{align*}
    J_\beta(\pi^\star) - J_\beta(\pihat)
\leq &~ 2 \beta \max_{\pi\in\set{\pi^\star,\pihat}}\sqrt{\E_\pi \brac{ f_{\htheta}(\tau) }}\leq 2\sqrt{H\Conchis{\mu} \paren{2\beta^2\epsopt+\epsstat(N)}}.
\end{align*}
This is the desired upper bound.
\qed

\subsection{Proof of \pref{lem:TBRM-stat}}

By \cref{asmp:theta-bound}, it holds that $\abs{f_\theta(\tau)}\leq B:=\frac{2C_\Theta +1}{\beta}$ for any $\theta\in\Theta$ and any trajectory $\tau$.
Using Freedman's inequality with the standard union bound, we have the following: \whp~(over the randomness of the dataset $\cD$), for all $\theta\in\Theta$,
\begin{align*}
    \EE_{\tau\sim \mu} f_\theta(\tau)^2 \leq 2\LTBRM(\theta)+ c_0B^2\paren{ \frac{\log \Ncov(\cF,\alpha)+\log(1/\delta)}{N}+B^{-1}\alpha },
\end{align*}
where $\cF=\set{f_\theta:\theta\in\Theta}$ is the function class induced by $\Theta$, and $c_0>0$ is an absolute constant. 

Next, for any fixed trajectory $\tau$, it holds that
\begin{align*}
    \nabla_\theta f_\theta(\tau)
    =&~\nabla_\theta\logit_\theta(s_1,a_1) + \sum_{h=2}^{H} \nabla_\theta\log\pi_\theta(a_{h} \mid s_{h}) \\
    =&~ \nabla_\theta\logit_\theta(s_1,a_1) + \sum_{h=2}^{H} \brac{ \nabla_\theta\logit_\theta(s_{h},a_h) - \EE_{a_h'\sim \pi_\theta(\cdot\mid s_h)}\nabla_\theta\logit_\theta(s_{h},a_h')  }.
\end{align*}
Hence, we can upper bound $\nrm{\nabla_\theta f_\theta(\tau)}_*\leq 2HL_\Theta$. This immediately implies that
\begin{align*}
    \linf{f_\theta-f_{\theta'}}=\sup_{\tau} \abs{ f_\theta(\tau)-f_{\theta'}(\tau) }\leq 2HL_\Theta\nrm{\theta-\theta'}, \qquad \forall \theta,\theta'\in\Theta.
\end{align*}
Therefore, we have
\begin{align*}
    \Ncov(\cF,\alpha)\leq \Ncov\paren{\Theta,\frac{\alpha}{2HL_\Theta}}.
\end{align*}
Combining the inequalities above and rescaling $\alpha\leftarrow 2HL_\Theta\alpha$ completes the proof.
\qed

%% file: sections/hard-inst.tex
\section{Hard Instances for Iterative Algorithms}
\label[appendix]{app:hard-instances}

In this section, we demonstrate the advantages of direct optimization (TBRM) over its iterative variant using a simple but illustrative hard instance.

By the autoregressive function approximation definition, we can rewrite the loss function of TBRM as follows:
\begin{align*}
\Lcal_{\Dcal}^{\sf TBRM}(\theta) = &~ \frac{1}{|\Dcal|}\sum_{\tau\in\Dcal} \left(\logit_\theta(s_1,a_1) - \frac{r(\tau)}{\beta} - \log\piref(\tau) + \sum_{h=2}^{H}\log\pi_\theta(a_{h} \mid s_{h}) \right)^2
\\
= &~ \frac{1}{|\Dcal|}\sum_{\tau\in\Dcal} \left(V_\theta(s_1) - \frac{r(\tau)}{\beta} + \sum_{h=1}^{H}\log\frac{\pi_\theta(a_{h} \mid s_{h})}{\piref(a_{h} \mid s_{h})} \right)^2.
\end{align*}

A typical iterative variant of this approach can be formulated as:
\begin{align}
    \label{eq:pcl-update}
    \theta_{t+1} \leftarrow \argmin_\theta \frac{1}{|\Dcal|} \sum_{\tau\in\Dcal} \left(V_{\theta_t}(s_1) - \frac{r(\tau)}{\beta} + \sum_{h=1}^{H}\log\frac{\pi_\theta(a_{h} \mid s_{h})}{\piref(a_{h} \mid s_{h})} \right)^2,
\end{align}
where $V_{\theta_t}$ is fixed from the previous iteration while optimizing for $\theta$.

To illustrate the difference between these approaches, we consider a simple 2-arm bandit problem where $r(a_1) = 1$ and $r(a_2) = 0$. We will show that, even at the population level, the iterative algorithm becomes trapped at a suboptimal solution, whereas TBRM converges to the globally optimal solution.

For this example, our Q-function class contains only two elements: $Q^\dagger = (0,0)$ and $Q^\star = (10,0)$, corresponding to the uniform policy and optimal policy, respectively.

Suppose at iteration $t$ we have $Q_t = Q^\dagger = (0,0)$, with temperature parameter $\beta = 0.1$, and data uniformly distributed over actions. The loss for the next iteration becomes:
\begin{align*}
\ell_t(Q) \coloneqq (1 - Q(a_1) + V_Q - V_t)^2 + (- Q(a_2) + V_Q - V_t)^2,
\end{align*}
where $V_t = \beta \log\sum_a\exp(\nicefrac{Q_t(a)}{\beta})$ represents the value function from the current iteration.

In this setting, we can verify that $\ell_t(Q^\dagger) < \ell_t(Q^\star)$, meaning the iterative algorithm will select $Q_{t+1} = Q_t = Q^\dagger = (0,0)$ and remain stuck at this suboptimal solution. In contrast, $Q^\star$ is the global minimizer of the TBRM loss by definition, demonstrating the advantage of direct optimization over the iterative approach.

%% file: sections/exp-appendix.tex
\section{Implementation Details}\label[appendix]{appendix:impl}

\subsection{Training Details}

To ensure rigorous and reproducible experimentation, we employ standardized and universally adopted hyperparameter settings, as detailed in \cref{sec:exp-main-result}. For baselines, we adhere closely to the recommended hyperparameter configurations as presented in VERL. Specifically, PPO training utilizes a learning rate of $1 \times 10^{-6}$ for the actor policy and $1 \times 10^{-5}$ for the critic policy. We incorporate a KL divergence coefficient and an entropy regularization coefficient of $0.001$ for PPO. The clip ratio for the actor loss function is set to $0.2$. For the GRPO baseline, we maintain the same KL divergence coefficient as PPO for the KL regularization term. To balance computational efficiency and performance, we utilize a prompt batch size of 128 and a maximum response length of 2048 tokens per training iteration. All generated responses are sampled using a temperature parameter of 1.0. For the TBRM method, the parameter $\beta$ is consistently set to $0.002$ across all experimental conditions. The learning rate for TBRM experiments is $2.5 \times 10^{-6}$, with the exception of the TBRM$_{n=1}$ with Qwen2.5-Math-7B model adopting a learning rate of $2 \times 10^{-6}$. All models are trained for a total of 760 steps.

All experiments are conducted on the same platform featuring 4x H100 80GB GPUs.

\subsection{Training Efficiency of TBRM}\label[appendix]{appendix:train-eff}

We compare the resource cost of TBRM, GRPO, and PPO in \cref{par:train-eff} by examining wall-clock time and maximal GPU memory usage. The labels for the time segments used in \cref{fig:train_eff} are detailed in \cref{tab:train-eff-label-desc}. To ensure a fair comparison, all experiments were conducted on the same platform featuring 4x H100 GPUs, and all configurations were standardized. Specifically, we employed vllm as the rollout backend and set \texttt{gpu\_memory\_utilization} to $0.4$. For policy updates, we set \texttt{micro\_batch\_size\_per\_gpu} to $1$, and for calculating log probabilities for both $\pi_{\theta_\mathsf{old}}$ and $\piref$, we also used a \texttt{micro\_batch\_size\_per\_gpu} of $1$. The value function model update for PPO also utilized a \texttt{micro\_batch\_size\_per\_gpu} of $1$. We use Qwen2.5-Math-7B as the base model, with a prompt batch size of $128$. 

\begin{table}[ht!]
    \centering
    \begin{tabular}{p{0.2\textwidth}m{0.5\textwidth}l}
        \toprule
        Name & Description & Involved Algorithms \\
        \midrule
        $\mathsf{ref\_prob}$ & Computing $\pi_{\text{ref}}(a_t\mid s_t)$ & TBRM, GRPO, PPO \\
        $\mathsf{old\_prob}$ & Computing $\pi_{\theta_{\text{old}}}(a_t\mid s_t)$ & GRPO, PPO \\
        $\mathsf{update\_policy}$ & Updating the policy parameter $\theta$ & TBRM, GRPO, PPO \\
        $\mathsf{ppo\_update\_value}$ & Updating the value function model $V_\phi$ & PPO \\
        $\mathsf{ppo\_value}$ & Computing $V_\phi (s_t)$ & PPO \\
        $\mathsf{rollout}$ & Sampling trajectories from the prompt set & TBRM, GRPO, PPO \\
        $\mathsf{other}$ &  Miscellaneous computations, e.g., rule-based reward $r(\tau)$, advantage (for GRPO and PPO), etc. &  \\
        \bottomrule
    \end{tabular}
    \label{tab:train-eff-label-desc}
    \caption{Description of labels of time segments in \cref{fig:train_eff}. }
\end{table}

\subsection{Prompt Templates}

\paragraph{Qwen-Math Prompt Template}
We use the default prompt template of Qwen2.5-Math in the main experiments (\cref{sec:exp-main-result}). 
\begin{tcolorbox}[colback=black!5, colframe=black!60, sharp corners, boxsep=2pt, top=2pt, bottom=2pt, left=4pt, right=4pt, boxrule=0.5pt, breakable]
\texttt{<|im\_start|>}system \\
Please reason step by step, and  put your final answer within \textbackslash{}boxed\{\}.\texttt{<|im\_end|>} \\
\texttt{<|im\_start|>}user \\
\textbf{question} \texttt{<|im\_end|>} \\
\texttt{<|im\_start|>}assistant
\end{tcolorbox}

\paragraph{DeepSeek-R1 Prompt Template}
We use DeepSeek-R1 prompt template in the experiment discussed in \cref{par:response_length}. 
\begin{tcolorbox}[colback=black!5, colframe=black!60, sharp corners, boxsep=2pt, top=2pt, bottom=2pt, left=4pt, right=4pt, boxrule=0.5pt, breakable]
A conversation between User and Assistant. The user asks a question, and the Assistant solves it. The Assistant first thinks about the reasoning process in the mind and then provides the user with the answer. The reasoning process and answer are enclosed within \texttt{<think> ... </think>} and \texttt{<answer> ... </answer>} tags, respectively, i.e., \texttt{<think>} reasoning process here \texttt{</think>} \texttt{<answer>} answer here \texttt{</answer>}. \\
User: \textbf{question} \\
Assistant:
\end{tcolorbox}

\section{Additional Experimental Results}\label[appendix]{appendix:additional-exp-result}

\subsection{Training-Time Performance}\label[appendix]{appendix:perbench_accuracy}

Following the setup described in \cref{sec:exp-setup}, we train and evaluate models of various sizes. Their performance on benchmarks during training is presented in \cref{fig:perbench_accuracy_15B} and \cref{fig:perbench_accuracy_7B}.

\begin{figure}
    \centering
    \includegraphics[width=1.0\linewidth]{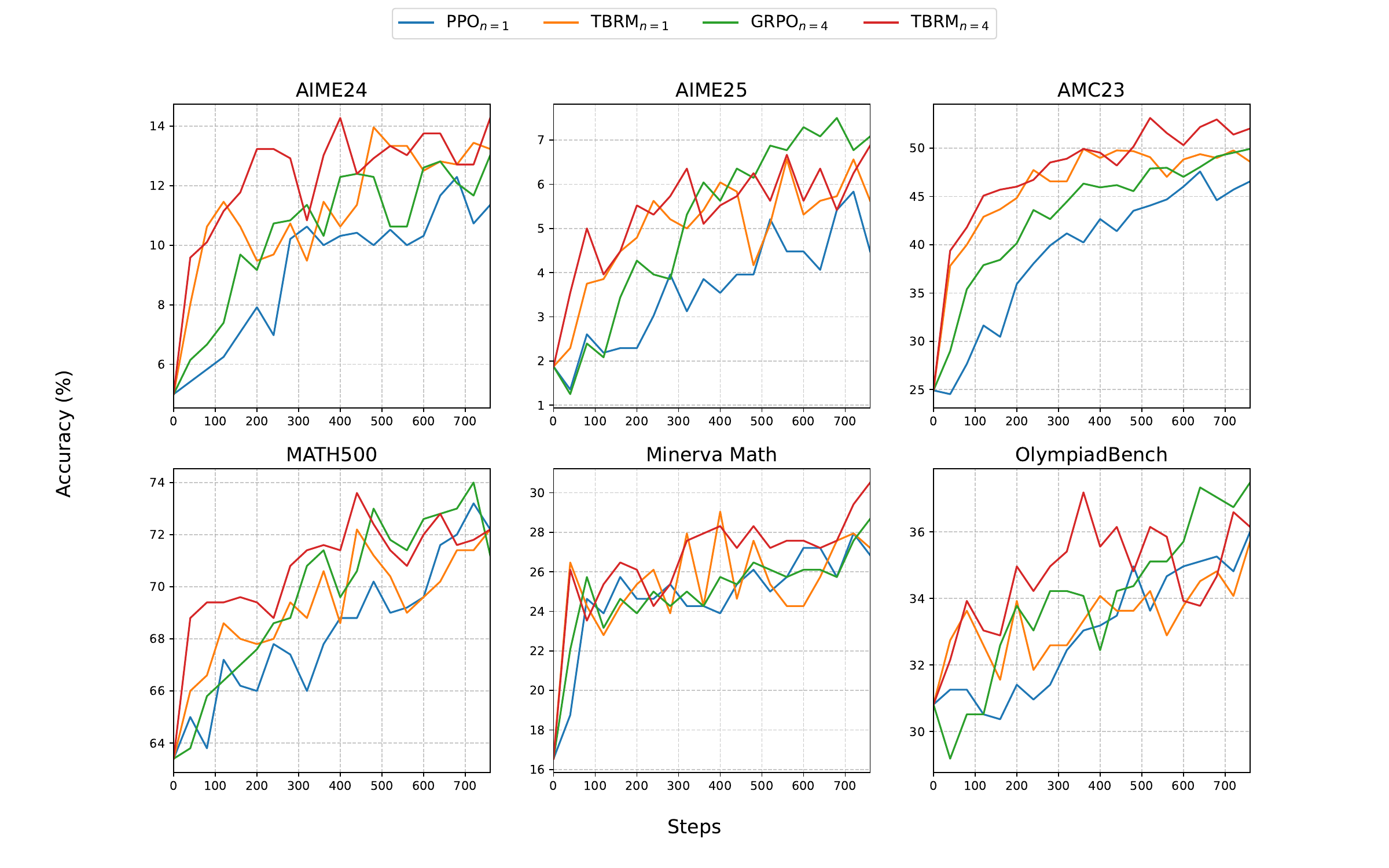}
    \caption{Learning accuracy across benchmarks with model Qwen2.5-Math-1.5B. }
    \label{fig:perbench_accuracy_15B}
\end{figure}

\begin{figure}
    \centering
    \includegraphics[width=1.0\linewidth]{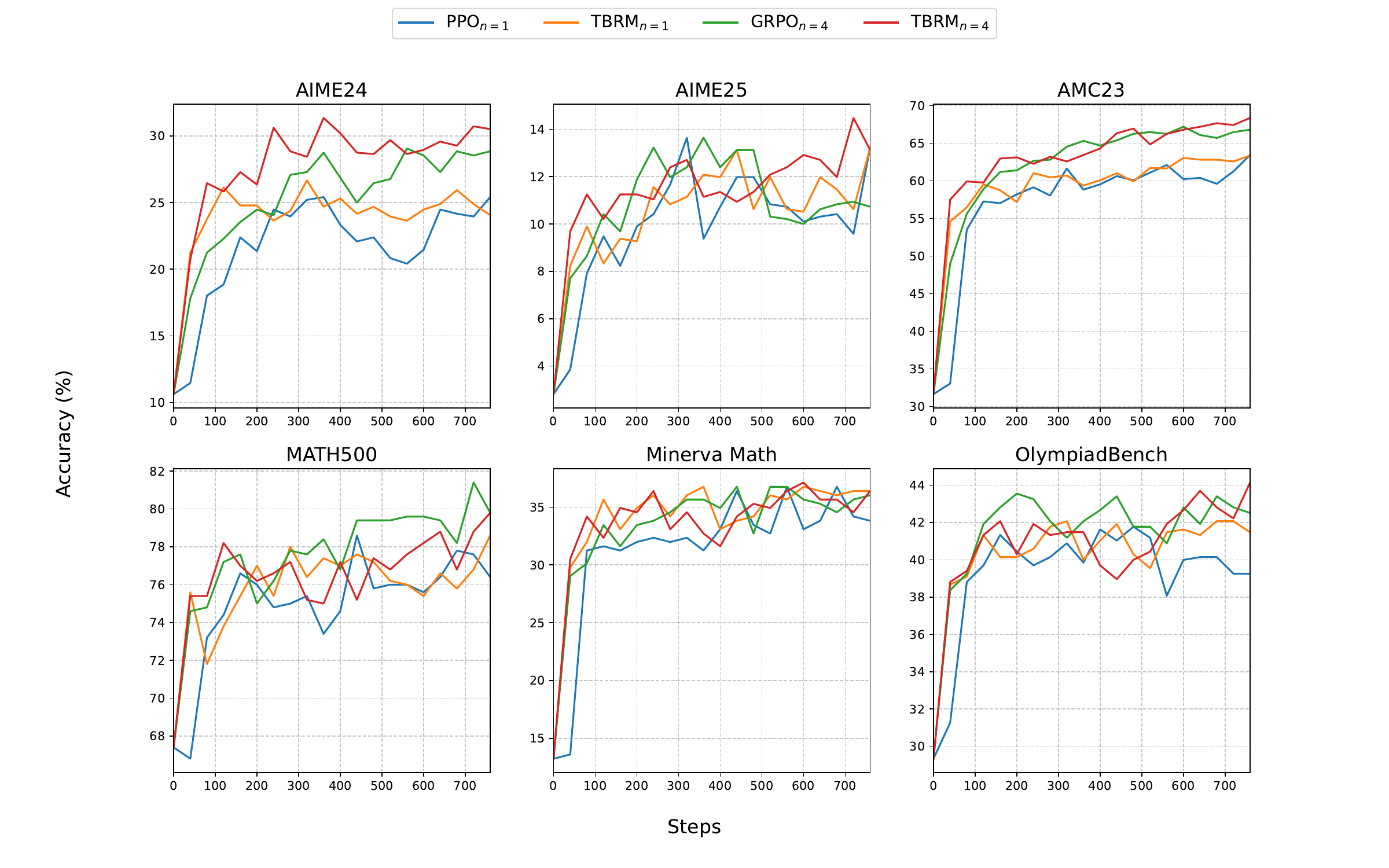}
    \caption{Learning accuracy across benchmarks with model Qwen2.5-Math-7B. }
    \label{fig:perbench_accuracy_7B}
\end{figure}

\subsection{Qualitative Analysis}\label[appendix]{appendix:qual-analysis}
We delve into the responses of the TBRM model and observe intriguing reasoning behaviors. In this subsection, we highlight emergent reasoning patterns that the model develops during training, which contribute to its improved reasoning capabilities. Unless otherwise noted, all responses presented here are sampled from Qwen2.5-Math-7B-TBRM$_{n=4}$.

\paragraph{Decomposition}
The model demonstrates the ability to break down complex problems into smaller, manageable subproblems, solving each sequentially to construct the final solution. This hierarchical strategy reflects a structured and efficient approach to reasoning.

\begin{tcolorbox}[breakable, colframe=gray, coltext=black, title=Decomposition, fonttitle=\bfseries, boxrule=0.5mm, width=\textwidth]

(MATH500) For what real values of $x$ is $-4<x^{4}+4x^{2}<21$ satisfied? Express your answer in interval notation.
\newline

{\color{blue}To solve the inequality $-4 < x^4 + 4x^2 < 21$, we need to break it into two parts and solve each part separately:

\begin{enumerate}
    \item Solve the inequality $x^4 + 4x^2 > -4$.
    \item Solve the inequality $x^4 + 4x^2 < 21$.
\end{enumerate}
}
\textbf{Step 1: Solve $x^4 + 4x^2 > -4$}

The expression $x^4 + 4x^2$ is always non-negative for all real numbers $x$ because both $x^4$ and $4x^2$ are non-negative. Therefore, $x^4 + 4x^2 \geq 0 > -4$ for all $x$. So, the inequality $x^4 + 4x^2 > -4$ is satisfied for all real numbers $x$.

\textbf{Step 2: Solve $x^4 + 4x^2 < 21$}

To solve $x^4 + 4x^2 < 21$, let's make a substitution. Let $y = x^2$. Then the inequality becomes $y^2 + 4y < 21$. Rearranging terms, we get $y^2 + 4y - 21 < 0$. We need to find the roots of the equation $y^2 + 4y - 21 = 0$ to determine the interval for $y$.

The quadratic equation $y^2 + 4y - 21 = 0$ can be factored as $(y + 7)(y - 3) = 0$. So, the roots are $y = -7$ and $y = 3$. Since $y = x^2$ and $x^2$ is always non-negative, we only consider the positive root $y = 3$. Thus, we need to solve $0 \leq y < 3$.

Substituting back $y = x^2$, we get $0 \leq x^2 < 3$. This means $- \sqrt{3} < x < \sqrt{3}$.

\textbf{Combining the Results}

Since the first inequality $x^4 + 4x^2 > -4$ is satisfied for all real numbers $x$, the solution to the original inequality $-4 < x^4 + 4x^2 < 21$ is determined by the second inequality. Therefore, the solution is $- \sqrt{3} < x < \sqrt{3}$.

The final answer in interval notation is:

$$
\boxed{(-\sqrt{3}, \sqrt{3})}
$$

\end{tcolorbox}

\paragraph{Enumeration}
The model adopts an exhaustive strategy by systematically iterating through possibilities or repeatedly applying a specific operation until the correct answer emerges. This approach is particularly useful in combinatorial or search-based problems.

\begin{tcolorbox}[breakable, colframe=gray, coltext=black, title=Enumeration, fonttitle=\bfseries, boxrule=0.5mm, width=\textwidth]

(AIME24) Alice chooses a set $A$ of positive integers. Then Bob lists all finite nonempty sets $B$ of positive integers with the property that the maximum element of $B$ belongs to $A$. Bob's list has 2024 sets. Find the sum of the elements of A.
\newline

To solve this problem, we need to understand the relationship between the set $A$ and the number of nonempty subsets $B$ of positive integers where the maximum element of $B$ belongs to $A$. 

If $A = \{a_1, a_2, \ldots, a_n\}$, then for each element $a_i \in A$, the number of subsets $B$ where the maximum element is $a_i$ is given by the number of nonempty subsets of $\{1, 2, \ldots, a_i - 1\}$, which is $2^{a_i-1} - 1$. Summing over all elements in $A$, the total number of such subsets $B$ is:
$$
    (2^{a_1} - 1) + (2^{a_2} - 1) + \ldots + (2^{a_n} - 1) = 2^{a_1} + 2^{a_2} + \ldots + 2^{a_n} - n
$$
Given that this total is 2024, we can write:
$$
    2^{a_1} + 2^{a_2} + \ldots + 2^{a_n} - n = 2024
$$
or
$$
    2^{a_1} + 2^{a_2} + \ldots + 2^{a_n} = 2024 + n
$$

To find the appropriate values for $a_1, a_2, \ldots, a_n$, we need to find a combination of powers of 2 that sum up to just over 2024. {\color{blue}We can start by finding the largest power of 2 less than or equal to 2024. The largest power of 2 less than 2024 is $2^{10} = 1024$. }

{\color{blue}By subtracting 1024 from 2024, we get 1000. The next largest power of 2 less than or equal to 1000 is $2^9 = 512$. Subtracting 512 from 1000 gives 488. The next largest power of 2 less than or equal to 488 is $2^8 = 256$, leaving us with 232. The next largest power of 2 less than or equal to 232 is $2^7 = 128$, leaving us with 104. The next largest power of 2 less than or equal to 104 is $2^6 = 64$, leaving us with 40. The next largest power of 2 less than or equal to 40 is $2^5 = 32$, leaving us with 8. The next largest power of 2 less than or equal to 8 is $2^3 = 8$, leaving us with 0.
So, the sum of the powers of 2 is:}
$$
    1024 + 512 + 256 + 128 + 64 + 32 + 8 = 2024
$$
This means that we have $ n = 7 $ elements in set $ A $, and these elements are \{11, 10, 9, 8, 7, 6, 4\} because the powers of 2 correspond to $2^{10}, 2^9, 2^8, 2^7, 2^6, 2^5, 2^3$.

So, the set $ A = \{11, 10, 9, 8, 7, 6, 4\} $. The sum of the elements of $ A $ is:
$$
    11 + 10 + 9 + 8 + 7 + 6 + 4 = 55
$$
The final answer is:
$$
    \boxed{55}
$$
\end{tcolorbox}

\paragraph{Verification}
The model exhibits self-monitoring behavior by verifying intermediate results during problem solving. This helps catch early mistakes and reinforces correct reasoning steps, enhancing both accuracy and reliability.

\begin{tcolorbox}[breakable, colframe=gray, coltext=black, title=Verification, fonttitle=\bfseries, boxrule=0.5mm, width=\textwidth]

(MATH500) What is the length, in units, of the radius of a sphere whose volume and surface area, in cubic units and square units, respectively, are numerically equal?
\newline

To find the radius of a sphere whose volume and surface area are numerically equal, we need to set up the equations for the volume and surface area of a sphere and then solve for the radius $ r $.

The volume $ V $ of a sphere is given by:
$ V = \frac{4}{3} \pi r^3 $

The surface area $ A $ of a sphere is given by:
$ A = 4 \pi r^2 $

We are given that the volume and surface area are numerically equal, so we can set the two equations equal to each other:
$ \frac{4}{3} \pi r^3 = 4 \pi r^2 $

To solve for $ r $, we can simplify the equation by dividing both sides by $ 4 \pi r^2 $ (assuming $ r \neq 0 $):
$ \frac{r}{3} = 1 $
$ r = 3 $

So, the radius of the sphere is $ r = 3 $. {\color{blue}Let's verify this by substituting $ r = 3 $ back into the volume and surface area formulas to ensure they are indeed equal.}

The volume when $ r = 3 $ is:
$ V = \frac{4}{3} \pi (3)^3 = \frac{4}{3} \pi \cdot 27 = 36 \pi $

The surface area when $ r = 3 $ is:
$ A = 4 \pi (3)^2 = 4 \pi \cdot 9 = 36 \pi $

Since the volume and surface area are both $ 36 \pi $, our solution is correct. The length of the radius of the sphere is $ \boxed{3} $.
\end{tcolorbox}

\paragraph{Backtracking}
Upon encountering contradictions or errors, the model is able to reassess its prior steps, revise incorrect assumptions, and pursue an alternative line of reasoning. This reflective capability mirrors human-like problem-solving resilience.

\begin{tcolorbox}[breakable, colframe=gray, coltext=black, title=Backtracking, fonttitle=\bfseries, boxrule=0.5mm, width=\textwidth]

(MATH500) The asymptotes of a hyperbola are $y = 2x - 3$ and $y = 17 - 2x.$ Also, the hyperbola passes through the point $(4,7).$ Find the distance between the foci of the hyperbola.
\newline

To find the distance between the foci of the hyperbola, we need to follow these steps:

\begin{enumerate}
    \item \textbf{Determine the center of the hyperbola:}
    The center of the hyperbola is the intersection point of the asymptotes. To find this point, we solve the system of equations given by the asymptotes:
    $$
    y = 2x - 3
    $$
    $$
    y = 17 - 2x
    $$
    Setting the two equations equal to each other:
    $$
    2x - 3 = 17 - 2x
    $$
    Solving for $x$:
    $$
    4x = 20 \implies x = 5
    $$
    Substituting $x = 5$ back into one of the equations to find $y$:
    $$
    y = 2(5) - 3 = 10 - 3 = 7
    $$
    So, the center of the hyperbola is $(5, 7)$.
    \item \textbf{Determine the slopes of the asymptotes:}
    The slopes of the asymptotes are $2$ and $-2$. This indicates that the hyperbola is oriented with its transverse axis vertical or horizontal. Since the slopes are $\pm 2$, the hyperbola is of the form:
    $$
    \frac{(x - h)^2}{b^2} - \frac{(y - k)^2}{a^2} = -1 \quad \text{or} \quad \frac{(y - k)^2}{a^2} - \frac{(x - h)^2}{b^2} = 1
    $$
    where the slopes of the asymptotes are $\pm \frac{a}{b}$ or $\pm \frac{b}{a}$. In this case, $\frac{a}{b} = 2$ or $\frac{b}{a} = 2$. Let's assume the hyperbola is of the form:
    $$
    \frac{(y - 7)^2}{a^2} - \frac{(x - 5)^2}{b^2} = 1
    $$
    with $\frac{a}{b} = 2$, so $a = 2b$.
    \item \textbf{Use the point (4, 7) to find $a$ and $b$:}
    The hyperbola passes through the point $(4, 7)$. Substituting $x = 4$ and $y = 7$ into the equation of the hyperbola:
    $$
    \frac{(7 - 7)^2}{a^2} - \frac{(4 - 5)^2}{b^2} = 1
    $$
    Simplifying, we get:
    $$
    0 - \frac{1}{b^2} = 1 \implies -\frac{1}{b^2} = 1 \implies b^2 = -1
    $$
    {\color{blue}This is incorrect, so let's assume the hyperbola is of the form:
    $$
    \frac{(x - 5)^2}{b^2} - \frac{(y - 7)^2}{a^2} = -1
    $$}
    with $\frac{a}{b} = 2$, so $a = 2b$. Substituting $x = 4$ and $y = 7$:
    $$
    \frac{(4 - 5)^2}{b^2} - \frac{(7 - 7)^2}{a^2} = -1
    $$
    Simplifying, we get:
    $$
    \frac{1}{b^2} - 0 = -1 \implies \frac{1}{b^2} = -1 \implies b^2 = 1
    $$
    So, $b = 1$ and $a = 2b = 2$.
    \item \textbf{Find the distance between the foci:}
    The distance between the foci of a hyperbola is $2c$, where $c = \sqrt{a^2 + b^2}$. In this case:
    $$
    c = \sqrt{a^2 + b^2} = \sqrt{2^2 + 1^2} = \sqrt{4 + 1} = \sqrt{5}
    $$
    So, the distance between the foci is:
    $$
    2c = 2\sqrt{5}
    $$
\end{enumerate}
The final answer is:
$$
\boxed{2\sqrt{5}}
$$

\end{tcolorbox}

\subsection{Ablation Study: Classical BRM on LLMs}\label[appendix]{app:classical-brm}

As discussed in \cref{sec:tbrm}, in the context of LLMs, directly minimizing BRM may not be possible because the token-level reward signal is either unavailable (e.g. if we assign the outcome reward to the final token) or very sparse. To highlight the advantages of the trajectory-level approach, we implemented classical BRM, whose loss function is defined in \cref{eq:classical_brm}, using the same hyperparameters as in our main setup on Qwen2.5-Math-1.5B. 

\begin{figure}[htbp]
    \centering
    \includegraphics[width=0.5\linewidth]{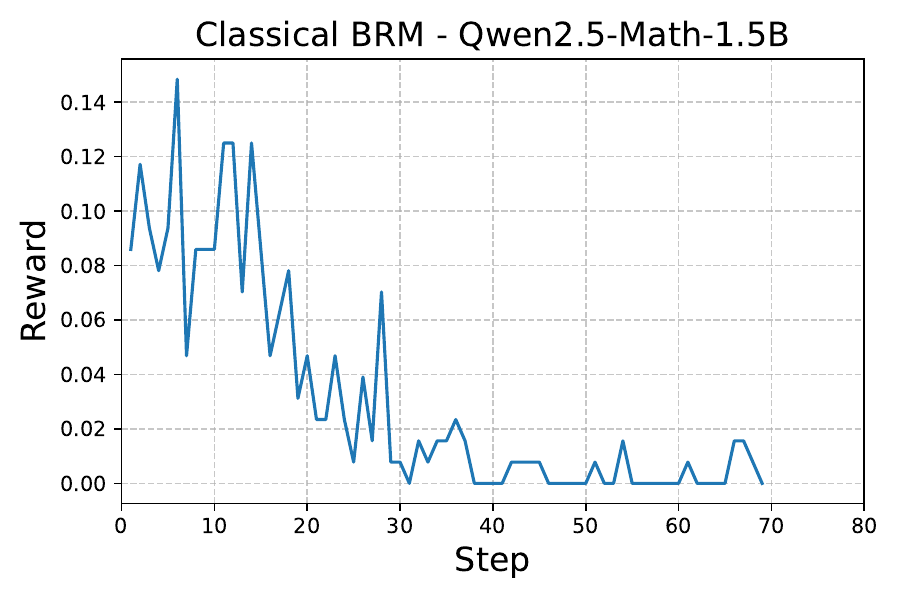}
    \caption{Training reward with classical BRM on model Qwen2.5-Math-1.5B. }
    \label{fig:statebrm}
\end{figure}

\cref{fig:statebrm} shows that training reward quickly collapses, and we observe that the model outputs become random and meaningless. Intuitively, this degradation occurs because BRM has to propagate the sparse reward signal, which only receives at the final token, back through multiple token-wise regressions, whereas the TBRM provides a better implicit credit assignment through a single trajectory-level objective.

\subsection{Tasks Beyond Mathematical Problems}\label[appendix]{app:reasoning-gym}

To demonstrate the generalizability of our method beyond mathematical tasks, we evaluate TBRM on five tasks from the reasoning-gym \citep{stojanovski2025reasoninggymreasoningenvironments} under the \textit{graphs} category: $\mathsf{course\_schedule}$, $\mathsf{family\_relationships}$, $\mathsf{largest\_island}$, $\mathsf{quantum\_lock}$, and $\mathsf{shortest\_path}$. These tasks are naturally represented as graphs, consisting of nodes and edges, and typically require traversing connections to identify relationships, compute optimal paths, or determine reachable components. They involve reasoning patterns that differ significantly from those in mathematical tasks.

\definecolor{bgcolor}{HTML}{D0E0F7}
\begin{table*}[ht!]
    \begin{adjustbox}{width=\textwidth,center}
    \centering
    \begin{tabular}{lcccccc}
        \toprule
        \textbf{Method} & \textbf{course\_schedule} & \textbf{family\_relationships} & \textbf{largest\_island} & \textbf{quantum\_lock} & \textbf{shortest\_path} & \textbf{Average} \\
        \midrule
        \textbf{Qwen2.5-Math-1.5B} & 29.5 & 3.0 & 11.0 & 5.5 & 0.0 & 9.8 \\
        \makebox[3em][l]{GRPO} \auxinfofont{$n=4$} & 54.0 & \textbf{84.0} & 34.0 & \textbf{30.5} & 26.0 & 45.7 \\
        \rowcolor{bgcolor} \makebox[3em][l]{TBRM} \auxinfofont{$n=4$} & \textbf{60.0} & 80.0 & \textbf{38.0} & 27.0 & \textbf{31.0} & \textbf{47.2} \\
        \bottomrule
    \end{tabular}
    \end{adjustbox}
    \caption{Performance of GRPO and TBRM on various tasks from reasoning-gym, category \textit{graphs}. }
    \label{tab:reasoninggym-results}
\end{table*}

We construct a training set of 10,000 problems, with 2,000 questions per task, and a test set of 500 problems, comprising 100 questions from each task. For both training and evaluation, we use the official verifiers provided by reasoning-gym to compute rewards. Our experiments are conducted on Qwen2.5-Math-1.5B using both TBRM and GRPO, with a prompt batch size of 1024 and 4 sampled responses per question ($n = 4$). Models are trained for 100 steps. All evaluations are conducted using greedy decoding. Results in \cref{tab:reasoninggym-results} demonstrate that TBRM generalizes well to diverse reasoning tasks and performs on par with GRPO.

\subsection{TBRM with More Rollouts}\label[appendix]{app:more-rollouts}

To demonstrate that TBRM scales effectively with increasing number of sampled responses per prompt, we rerun GRPO and TBRM using most hyperparameters from DAPO \citep{yu2025dapo}. Specifically, we used a prompt batch size of 512 and generated $n=16$ responses per prompt. For GRPO, we set the microbatch size to 512, resulting in 16 updates per training step. The experiments were conducted on the Qwen2.5-Math-7B model, following the same evaluation pipeline described in our paper. Both algorithms were trained for 100 steps. \cref{tab:more-rollouts-results} shows that TBRM remains comparable to GRPO under these aligned settings.

\definecolor{bgcolor}{HTML}{D0E0F7}
\begin{table*}[ht!]
    \begin{adjustbox}{width=\textwidth,center}
    \centering
    \begin{tabular}{lcccccc}
        \toprule
        \textbf{Method} & \textbf{AIME24} & \textbf{AIME25} & \textbf{AMC23} & \textbf{MATH500} & \textbf{Minerva Math} & \textbf{OlympiadBench} \\
        & \auxinfofont{Avg@32} & \auxinfofont{Avg@32} & \auxinfofont{Avg@32} & \auxinfofont{Avg@1} & \auxinfofont{Avg@1} & \auxinfofont{Avg@1} \\
        \midrule
        \textbf{Qwen2.5-Math-7B} & 10.6 & 2.8 & 31.6 & 67.4 & 13.2 & 29.3 \\
        \makebox[3em][l]{GRPO} \auxinfofont{$n=16$} & 26.6 & \textbf{11.0} & 61.8 & \textbf{77.8} & 32.7 & \textbf{40.4} \\
        \rowcolor{bgcolor} \makebox[3em][l]{TBRM} \auxinfofont{$n=16$} & \textbf{27.9} & 10.9 & \textbf{62.8} & 76.4 & \textbf{33.5} & 39.9 \\
        \bottomrule
    \end{tabular}
    \end{adjustbox}
    \caption{Performance of GRPO and TBRM with $n=16$. }
    \label{tab:more-rollouts-results}
\end{table*}